%% file: main.tex
\documentclass{article}
\usepackage{iclr2027_conference,times}

\input{math_commands.tex}

\usepackage[utf8]{inputenc} 
\usepackage[T1]{fontenc}    
\usepackage{amsmath,amsthm}
\usepackage{amssymb}
\usepackage{booktabs}
\usepackage{graphicx}
\usepackage[
  pagebackref=true,
  colorlinks=true,
  linkcolor=Blue,
  citecolor=Blue,
  urlcolor=Blue
]{hyperref}       
\usepackage{url}
\usepackage{subcaption}
\usepackage{placeins}
\usepackage{algorithm}
\usepackage{algpseudocode}
\usepackage{enumitem}
\usepackage{xspace}
\usepackage[dvipsnames]{xcolor}         
\algrenewcommand\algorithmicrequire{\textbf{Input:}}
\algrenewcommand\algorithmicensure{\textbf{Output:}}

\title{\oursshort{}: Hadamard Codebooks for Fast Generative Image and Video Compression}

\author{\stepcounter{footnote}Or Ronai\thanks{Equal contribution.} , 
Vladimir Kulikov\footnotemark[2] , 
Tomer Michaeli \\
Technion -- Israel Institute of Technology \\
\texttt{\{or.ronai@campus, vladimir.k@campus, tomer.m@ee\}.technion.ac.il} \\
Project page: \url{https://h-ddcm.github.io}
}

\newcommand{\oursshort}{H-DDCM}
\newcommand{\myparagraph}[1]{\vspace{-5pt}\paragraph{#1}}

\newcommand{\mySuppTitle}{
    {\vbox{\hsize\textwidth
    {\LARGE\sc 
    Supplementary Material\par}
    \vskip 0.3in minus 0.1in}}
    \vspace{-1.5em}
}

\makeatletter
\DeclareRobustCommand\onedot{\futurelet\@let@token\@onedot}
\def\@onedot{\ifx\@let@token.\else.\null\fi\xspace}
\def\eg{\emph{e.g}\onedot, }

\def\ie{\emph{i.e}\onedot, }

\makeatother

\makeatletter
\newcommand*{\addFileDependency}[1]{
  \typeout{(#1)}
  \@addtofilelist{#1}
  \IfFileExists{#1}{\typeout{File #1 O.K.}}{\typeout{No file #1.}}
}
\makeatother

\newcommand{\beginsupplement}{%
    \renewcommand\thefigure{S\arabic{figure}}    
    \setcounter{figure}{0}
    \renewcommand\thetable{S\arabic{table}}    
    \setcounter{table}{0}  
    \renewcommand\theequation{S\arabic{equation}}    
    \setcounter{equation}{0}
    \renewcommand\thealgorithm{S\arabic{algorithm}}    
    \setcounter{algorithm}{0}  
}

\theoremstyle{plain}
\newtheorem{proposition}{Proposition}

\theoremstyle{definition}

\theoremstyle{remark}

\setlist[enumerate]{leftmargin=1.6em, labelsep=.4em}
\setlist[itemize]{leftmargin=1.6em, labelsep=.4em}

\iclrfinalcopy
\begin{document}

\maketitle

\begin{abstract}
\input{sections/abstract}
\end{abstract}

\input{sections/introduction}
\input{sections/related_work}
\input{sections/background}
\input{sections/method}
\input{sections/experiments}
\input{sections/application_to_streaming}
\input{sections/conclusion}
\input{sections/ai_use}
\input{sections/ethics_reproducibility}

\bibliography{iclr2027_conference}
\bibliographystyle{iclr2027_conference}

\clearpage
\appendix
\input{supplementary/supplementary}

\end{document}

%% file: math_commands.tex
\usepackage{amsmath,amsfonts,bm}

\def\eqref#1{(\ref{#1})}

\def\1{\bm{1}}

\DeclareMathAlphabet{\mathsfit}{\encodingdefault}{\sfdefault}{m}{sl}
\SetMathAlphabet{\mathsfit}{bold}{\encodingdefault}{\sfdefault}{bx}{n}

\DeclareMathOperator*{\argmin}{arg\,min}

\DeclareMathOperator{\sign}{sign}

\DeclareMathOperator{\TopM}{Top\text{-}M}
\DeclareMathOperator{\diag}{diag}
\DeclareMathOperator{\std}{std}

%% file: sections/abstract.tex
Training-free diffusion-based compression methods repurpose the generative priors of a diffusion or a flow model to compress data without retraining them as task-specific codecs. Recent codebook-based approaches such as Turbo-DDCM achieve strong perceptual quality at ultra-low bitrates, but their reliance on dense Gaussian codebooks remains costly in both memory and computation, and scales poorly with the large latent dimensions of modern image and video generators. Their combinatorial bitstream encoding further becomes computationally heavy at this scale. We introduce \oursshort{}, a scalable generative compression framework that replaces Gaussian codebooks with implicit randomized Hadamard dictionaries. The fast Walsh-Hadamard transform allows for implicitly synthesizing full-dimensional dictionaries with $O(d\log d)$ compute for projection and reconstruction and $O(d)$ memory, without materializing the full-dimensional dictionary. We further introduce an efficient table-free lexicographic coding scheme that makes combinatorial encoding practical at larger dimensions. Beyond images, we adapt the framework to an autoregressive video flow model and develop a streaming compression scheme that periodically re-encodes recently reconstructed frames into a fresh KV-cache prefix, preserving temporal context across segments at no additional transmitted bits. This enables long-form, close to real-time generative video compression on a single GPU. Across image and video benchmarks, \mbox{\oursshort{}} achieves competitive rate-distortion-perception performance while substantially reducing the memory and computational overhead of codebook-based training-free generative compression.

%% file: sections/introduction.tex
\section{Introduction}
\label{sec:introduction}

Generative priors have become increasingly important for low-bitrate image and video compression, where reconstructing perceptually plausible content can be preferable to preserving every pixel accurately. Early generative codecs were based on GANs~\citep{goodfellow2020generative,agustsson2019generative,mentzer2020high}, while more recent methods exploit the stronger priors of diffusion~\citep{sohl2015deep,ho2020denoising,song2021scorebased} and flow models~\citep{lipman2023flow,liu2023flow,albergo2025stochastic}. Most such approaches train or fine-tune the generative model specifically for compression~\citep{yang2023lossy,careil2024towards,korber2024perco,zhang2026regenvc,mao2026generative}, tying the codec to a particular backbone. In contrast, training-free diffusion compression leaves a pretrained generator frozen and encodes information directly into its sampling trajectory. Denoising Diffusion Codebook Models (DDCM)~\citep{ohayon2025compressed} realize this idea by replacing the random noise at each denoising step with a noise map selected from a $K$-atom Gaussian codebook shared by the encoder and decoder. The encoder selects the codeword that best aligns with the residual between the target and the model’s current clean prediction and transmits only its index; the decoder regenerates the same codeword and thereby follows the encoded denoising trajectory. Turbo-DDCM~\citep{vaisman2026turbo} substantially accelerates this procedure by selecting and combining multiple codebook atoms in closed form at each step, producing a noise map better aligned with the residual and cutting the expensive denoising steps by over $90\%$.

Despite this acceleration, Turbo-DDCM remains difficult to scale to the high-dimensional latent spaces of modern generative models. Its Gaussian codebook must either be stored or regenerated, and residual matching requires dense operations whose cost grows with both the latent dimension and codebook size. Restricting the codebook reduces this cost but also limits the set of residual directions available to the encoder. A complete codebook becomes particularly problematic at modern scales, \ie for large image latents such as FLUX~\citep{flux2024}, it already requires approximately $128$\,GiB of storage, while video latents such as those used by Causal-Forcing (CF)~\citep{zhu2026causal} require roughly $167$\,GiB. These costs turn the codebook itself into a major computational and memory bottleneck, limiting the applicability of training-free generative compression precisely for the large image and fast video backbones where it is most attractive.

In this work, we address these limitations through three main contributions. First, we replace dense Gaussian codebooks with structured randomized Hadamard dictionaries. Turbo-DDCM only accesses its codebook to correlate candidate atoms with the current residual and to reconstruct the selected combination; both operations can therefore be replaced by fast matrix-vector products. A randomized Hadamard transform provides these operations implicitly, without materializing the dictionary, reducing the computation to $O(d\log d)$ and memory to $O(d)$. Moreover, while Turbo-DDCM relies on the approximate orthogonality of high-dimensional Gaussian atoms, the Hadamard basis is exactly orthogonal, making its top-$M$ candidate thresholding rule optimal for the orthogonal Hadamard representation. We further extend this construction to an overcomplete structured dictionary, combining multiple independently randomized Hadamard bases to expose an even richer set of candidate directions while retaining the fast implicit projections.

Second, we address a separate computational bottleneck in the transmitted bitstream. Turbo-DDCM exploits the fact that the $M$ selected codebook indices form an unordered subset and encodes this subset through a lexicographic rank, reducing the index cost from independently transmitting $M$ indices to $\lceil \log_2 \tbinom{d}{M} \rceil$ bits. However, its implementation relies on a dense precomputed combinatorial table whose construction and memory requirements become impractical at the higher dimensions considered here. We develop a table-free lexicographic encoder and decoder, together with a grouping protocol that bounds the size of individual combinatorial operations. This retains the bit savings of the combinatorial coding while ensuring practical runtime at large latent dimensions.

Finally, we extend the framework to close to real-time autoregressive (AR) video compression. We adapt Causal-Forcing, a fast causal AR video flow model, into a streaming generative codec.
To additionally compress beyond the model’s native context length without periodically discarding temporal context, we introduce a re-encoded-tail cache refresh: at each continuation boundary, recently reconstructed frames are re-encoded into a fresh KV-cache prefix, providing the next segment with temporal context at no additional transmitted bits. Together, these components enable scalable streaming compression on a single GPU. On images, our structured dictionary matches the fidelity of complete-codebook Turbo-DDCM without materializing the corresponding codebook; on video, our streaming codec achieves competitive rate-distortion-perception performance while providing the highest measured throughput among the evaluated methods sharing the CF backbone.

%% file: sections/related_work.tex
\vspace{-2pt}
\section{Related work}
\vspace{-2pt}
\label{sec:related}

\paragraph{Neural and generative compression.}
Learned codecs replace the transform and entropy model of classical compression with trained networks~\citep{balle2017end,toderici2017full}. These approaches were extended from images to video via learned motion estimation and residual coding~\citep{lu2019dvc,li2023neural}. The original methods targeted low distortion, but tended to produce over-smoothed reconstructions at low bitrates. Later generative codecs traded exact fidelity for plausibility, first with GANs~\citep{agustsson2019generative,mentzer2020high} and more recently with diffusion priors trained from scratch~\citep{yang2023lossy,ghouse2023residual}, fine-tuned for compression~\citep{careil2024towards,korber2024perco}, or distilled into a single reverse step~\citep{park2026single,xue2026one}. This framework has extended to video, where a pretrained video diffusion model refines the output of a separately designed codec~\citep{mao2026generative}. All of these require training or fine-tuning dedicated to compression, which ties the codec to a particular backbone and generally requires retraining when the backbone changes.

\myparagraph{Training-free diffusion-based image compression.}
Training-free methods instead leave the generative model frozen and encode information through its sampling trajectory. DiffC~\citep{theis2022lossy} communicates samples of the forward process through reverse channel coding, which~\citet{vonderfecht2025lossy} made practical with a custom CUDA kernel. IPIC~\citep{xu2024idempotence} uses an off-the-shelf optimized encoder and builds on diffusion-based posterior sampling for decoding, while PSC~\citep{elata2025psc} uses a posterior-sampling based adaptive compressed sensing method~\citep{donoho2006compressed,elata2024adaptive} for encoding with a pretrained diffusion model. DDCM~\citep{ohayon2025compressed} replaces the Gaussian noise at each reverse step with an atom selected from a reproducible codebook, and Turbo-DDCM~\citep{vaisman2026turbo} combines multiple atoms per step in closed form, reducing the number of expensive denoising evaluations by over $90\%$. Our method belongs to this training-free setting, but replaces the explicit Gaussian codebooks with an implicit structured dictionary.

\myparagraph{Training-free generative video compression.}
Carrying the idea to video raises a difficulty absent from images, since temporal coherence is a property of the joint distribution across frames rather than of each frame separately.
ZeroGVC~\citep{gao2026zerogvc} encodes the first frame of every group of pictures as an intra-frame (I-frame) using an off-the-shelf image codec, \ie OneDC~\citep{xue2026one}, and incorporates Turbo-DDCM into an AR model for the remaining predicted frames (P-frames), steering a few-step consistency trajectory toward each target frame, with an optional bidirectional mode that borrows context from the next I-frame at no extra rate.
GVCC~\citep{zeng2026gvcc} observes that modern video foundation models are sampled by a deterministic rectified-flow ODE, which leaves no per-step stochastic channel to encode into, and restores one by converting the sampler into an equivalent marginal-preserving SDE. Both retain an explicit Gaussian codebook, and video is where that choice is most expensive: a latent chunk is larger than an image latent, and the codebook grows with the product of its size and that dimension (App.~\ref{app:codebook_cost}).

\myparagraph{Autoregressive and streaming video diffusion.}
Recent advances in video generation increasingly utilize causal and AR formulations to support real-time generation. Diffusion Forcing~\citep{chen2024diffusion} assigns independent noise levels per frame and thereby supports AR rollout, while CausVid~\citep{yin2025slow} and Self-Forcing~\citep{huang2026self} distill bidirectional video diffusion models into few-step causal students. Addressing structural limits in this distillation, \citet{zhu2026causal} demonstrate that the ODE-initialization stage requires frame-level injectivity -- a condition violated by bidirectional teachers. Their Causal-Forcing (CF) method resolves this by distilling from an AR teacher, thereby achieving practical real-time streaming generation on a single GPU.

%% file: sections/background.tex
\vspace{-2pt}
\section{Background}
\vspace{-2pt}
\label{sec:background}

Our method builds directly on Turbo-DDCM, which compresses data by replacing the random noise along a diffusion trajectory with combinations of atoms from a shared Gaussian codebook. We first review this formulation and its sparse noise-selection rule, and highlight the computational, memory and bitstream bottlenecks that emerge as latent dimension grows. We then introduce the randomized Hadamard transform and the properties that allow us to replace these dense codebook operations with fast implicit projections.

\vspace{-2pt}
\subsection{Denoising diffusion models}
\vspace{-2pt}
\label{sec:bg:ddcm}

A diffusion model~\citep{sohl2015deep,ho2020denoising,song2021scorebased} generates data by reversing a gradual noising process. Starting from $\mathbf{x}_T \sim \mathcal{N}(\mathbf{0},\mathbf{I}_d)$, the reverse step is given by 
\begin{equation}
\label{eq:ddpm}
\mathbf{x}_{t-1} = \boldsymbol{\mu}_t(\mathbf{x}_t) + \sigma_t\mathbf{z}_t, \qquad \mathbf{z}_t \sim \mathcal{N}(\mathbf{0}, \mathbf{I}_d),
\end{equation}
Here, $\boldsymbol{\mu}_t$ is parametrized by a neural network and $\mathbf{z}_t$ provides the stochastic degree of freedom in the transition. Its exact parameterization depends on the sampler: we use DDPM sampling for images and a few-step consistency formulation for our causal video backbone, whose corresponding updates are detailed in App.~\ref{app:sampling_formulations}. DDCM exploits precisely this freedom in $\mathbf{z}_t$, replacing randomly sampled noise with information-carrying noise maps selected by the encoder.

\vspace{-2pt}
\subsection{Turbo-DDCM}
\vspace{-2pt}
\label{sec:bg:turbo}

DDCM~\citep{ohayon2025compressed} replaces the random noise in~\eqref{eq:ddpm} with noise maps selected from a reproducible codebook $\boldsymbol{C}_t\in\mathbb{R}^{d\times K}$. Its columns $\smash{\mathbf{z}_t^{(1)},\ldots,\mathbf{z}_t^{(K)}}$ are i.i.d.\ Gaussian atoms generated from a seed shared by the encoder and decoder. To compress an image $\mathbf{x}_0$, the encoder selects in each timestep $t$ the atom most aligned with the residual $\mathbf{r}_t
:=
\mathbf{x}_0-\hat{\mathbf{x}}_{0|t}$, where $\hat{\mathbf{x}}_{0|t}$ is the model’s current clean prediction, which can be obtained from $\boldsymbol{\mu}_t(\mathbf{x}_t)$. Injecting a noise map correlated with $\mathbf{r}_t$ steers the sampling trajectory toward the target, while transmitting only its codebook index allows the decoder to reproduce the same trajectory.

Selecting a single atom provides limited residual correction per denoising step. Turbo-DDCM~\citep{vaisman2026turbo} instead approximates $\mathbf{r}_t$ using a sparse combination of $M$ codebook atoms, allowing substantially more information to be transmitted per step, therefore reducing the number of expensive denoiser evaluations. Writing $\boldsymbol{\alpha}_t \in \mathbb{R}^{K}$ for the coefficient vector, the per-step problem is
\begin{equation}
\label{eq:sparse-ls}
\boldsymbol{\alpha}_t^{\star} = \argmin_{\boldsymbol{\alpha} \in \mathbb{R}^{K}} \left\| \boldsymbol{C}_t\boldsymbol{\alpha} - \mathbf{r}_t \right\|_2^2 \quad \text{s.t.} \quad \|\boldsymbol{\alpha}\|_0 = M, \quad \alpha_i \in \mathcal{V} \cup \{0\} \ \ \forall i,
\end{equation}
with $\mathcal{V}$ being a quantization set of $2^{C}$ levels, so each nonzero coefficient costs $C$ bits.
Problem~\eqref{eq:sparse-ls} is combinatorial for a general dictionary. However, if the dictionary has orthonormal columns, it separates across coordinates and is solved exactly by thresholding its correlations with the residual~\citep{elad2010sparse}. Gaussian atoms are not orthogonal, but become approximately so in high dimensions; Turbo-DDCM therefore applies the same thresholding rule. For the one-bit case $\mathcal{V}=\{-1,+1\}$,
\begin{equation}
\label{eq:thresholding}
(\boldsymbol{\alpha}_t^{\star})_i =
\begin{cases}
\sign\big((\mathbf{u}_t)_i\big), & i \in \TopM(|\mathbf{u}_t|),\\ 0, & \text{otherwise},
\end{cases}
\qquad \mathbf{u}_t = \boldsymbol{C}_t^\top \mathbf{r}_t .
\end{equation}
For an orthogonal dictionary, this is exactly the minimizer of~\eqref{eq:sparse-ls}: the support collects the $M$ largest correlations, and each selected coefficient is rounded to the closest value in $\mathcal{V}$. The resulting noise map is normalized, and injected in place of $\mathbf{z}_t$ in~\eqref{eq:ddpm}
\begin{equation}
\label{eq:std-norm}
\mathbf{z}_t^\star =
\frac{\boldsymbol{C}_t\boldsymbol{\alpha}_t^\star}{\std(\boldsymbol{C}_t\boldsymbol{\alpha}_t^\star)}.
\end{equation}

Turbo-DDCM further reduces bitrate by terminating codebook selection early and completing the trajectory with $N+1$ deterministic DDIM steps~\citep{song2021denoising}. It also exploits the fact that the $M$ selected indices form an unordered subset: rather than transmitting them independently, it sends their lexicographic rank using $\lceil\log_2\binom{K}{M}\rceil + MC$ bits per coded step. This avoids the permutation redundancy of independently listing the selected indices.

Despite these improvements, two computational bottlenecks remain and become severe at large latent dimensions. First, computing $\boldsymbol{C}_t^\top\mathbf{r}_t$ and reconstructing $\boldsymbol{C}_t\boldsymbol{\alpha}_t^\star$ require access to a dense $d\times K$ Gaussian codebook, whose memory and computation grow with both $d$ and $K$. Practical implementations must therefore restrict $K$ as the latent dimension grows, reducing the set of directions available to represent the residual. Second, Turbo-DDCM’s lexicographic rank and unrank operations rely on a dense precomputed table of binomial coefficients for each $(K,M)$ pair, which becomes expensive to construct and store at large $K$. Our method targets both scaling bottlenecks directly.

\vspace{-2pt}
\subsection{Randomized Hadamard transform}
\vspace{-2pt}
\label{sec:bg:hadamard}

The dense-codebook bottleneck above motivates a dictionary that is both orthogonal and admits fast implicit matrix-vector products.
Let $d$ be a power of two and $\mathbf{H}\in\mathbb{R}^{d\times d}$ denote the normalized Hadamard matrix, via Sylvester’s recursion and scaled by $\smash{d^{-1/2}}$. Its entries are $\smash{\pm d^{-1/2}}$ and
\begin{equation}
\label{eq:hadamard}
\mathbf{H} = \mathbf{H}^\top, \qquad \mathbf{H}^\top\mathbf{H} = \mathbf{I}_d, \qquad \mathbf{H}^2 = \mathbf{I}_d .
\end{equation}
Thus, the columns of $\mathbf{H}$ form an orthonormal basis of $\mathbb{R}^d$. Crucially, this basis never needs to be materialized: the fast Walsh-Hadamard transform (FWHT)~\citep{fino1976unified} computes $\mathbf{H}\mathbf{v}$, \ie the correlations of $\mathbf{v}$ with all Hadamard basis vectors simultaneously, using $O(d\log d)$ computation and $O(d)$ memory.

To randomize the basis, let $\boldsymbol{\varepsilon}\in\{-1,+1\}^d$ contain i.i.d.~Rademacher samples, define $\mathbf{D}=\diag(\boldsymbol{\varepsilon})$, and consider the dictionary $\mathbf{D}\mathbf{H}$ rather than $\mathbf{H}$. On the analysis side, multiplying a vector by the transposed matrix $\mathbf{H}\mathbf{D}$ uniformly disperses its energy across all coordinates and reduces coherence with structured inputs~\citep{tropp2011improved}. Since $\mathbf{D}^\top\mathbf{D}=\mathbf{I}_d$,
\begin{equation}
\label{eq:DH-orthonormal}
(\mathbf{D}\mathbf{H})^\top(\mathbf{D}\mathbf{H}) = \mathbf{H}^\top\mathbf{D}^\top\mathbf{D}\mathbf{H} = \mathbf{H}^\top\mathbf{H} = \mathbf{I}_d,
\end{equation}
so that the randomized basis $\mathbf{D}\mathbf{H}$ remains exactly orthonormal, while multiplication by either $\mathbf{D}\mathbf{H}$ or its transpose requires only one FWHT and sign modulation. It can consequently be regenerated from a shared seed and applied with $O(d\log d)$ computation and $O(d)$ memory, without storing any of its $\smash{d^2}$ entries.

%% file: sections/method.tex
\vspace{-2pt}
\section{Method}
\vspace{-2pt}
\label{sec:method}

We address both scaling bottlenecks from Sec.~\ref{sec:background}: the dense Gaussian codebook and the combinatorial bitstream computation cost. First, we replace the Gaussian codebook with an implicit randomized Hadamard dictionary, enabling fast residual projection and reconstruction without materializing the dictionary (Sec.~\ref{sec:method:selection}). Then, we introduce a scalable table-free, grouped lexicographic bitstream implementation (Sec.~\ref{sec:method:rate}). Finally, we extend the structured dictionary with multiple randomized Hadamard bases to increase the available candidate directions in the low-rate regime (Sec.~\ref{sec:method:overcomplete}).

\vspace{-2pt}
\subsection{Structured Hadamard dictionary}
\vspace{-2pt}
\label{sec:method:selection}

Turbo-DDCM accesses its codebook $\boldsymbol{C}_t$ in two places -- forming the correlations $\mathbf{u}_t = \boldsymbol{C}_t^\top\mathbf{r}_t$ and reconstructing $\boldsymbol{C}_t\boldsymbol{\alpha}_t^\star$. Both require only multiplication by the codebook and its transpose, which the Hadamard transform (Sec.~\ref{sec:bg:hadamard}) provides implicitly without materializing the dictionary.

To simplify the exposition, assume that the dimension $d$ of the residual $\mathbf{r}_t$ is a power of two (otherwise it is zero-padded to the closest power of two).
At denoising step $t$, we generate a Rademacher diagonal matrix $\mathbf{D}_t$ from the seed shared by the encoder and decoder and define
\begin{equation}
\label{eq:dictionary}
\mathbf{A}_t := \mathbf{D}_t\mathbf{H} \in \mathbb{R}^{d \times d}
\end{equation}
whose columns form an orthonormal basis by~\eqref{eq:DH-orthonormal}. A new $\mathbf{D}_t$ is generated at every step, providing a different randomized basis along the trajectory without storing or transmitting the dictionary. We approximate $\mathbf{r}_t$ in this basis by
\begin{equation}
\label{eq:our-problem}
\boldsymbol{\alpha}_t^{\star} = \argmin_{\boldsymbol{\alpha} \in \mathbb{R}^{d}} \left\| \mathbf{r}_t - \mathbf{A}_t\boldsymbol{\alpha} \right\|_2^2 \quad \text{s.t.} \quad \|\boldsymbol{\alpha}\|_0 = M, \quad \alpha_i \in \mathcal{V} \cup \{0\} \ \ \forall i,
\end{equation}
which is the sparse least-squares equation from~\eqref{eq:sparse-ls} but with the Gaussian codebook $\boldsymbol{C}_t$ replaced by a randomized Hadamard codebook $\mathbf{A}_t$, with $K = d$.
Because $\mathbf{A}_t$ is an orthonormal basis, the objective decouples coordinate-wise via
\begin{equation}
\label{eq:separable}
\left\|\mathbf{r}_t - \mathbf{A}_t\boldsymbol{\alpha}\right\|_2^2 = \left\|\mathbf{A}_t(\mathbf{u}_t - \boldsymbol{\alpha})\right\|_2^2 = \left\|\mathbf{u}_t - \boldsymbol{\alpha}\right\|_2^2, \qquad \mathbf{u}_t = \mathbf{A}_t^\top \mathbf{r}_t.
\end{equation}
Consequently, the thresholding rule of~\eqref{eq:thresholding}, which was only an approximation for the nearly orthogonal Gaussian atoms of Turbo-DDCM, solves~\eqref{eq:our-problem} exactly: the support is $\mathcal{I}_t := \TopM(|\mathbf{u}_t|)$ and each selected entry is rounded to the nearest level in $\mathcal{V}$~\citep{vaisman2026turbo}.
The selected combination is rescaled and injected in place of $\mathbf{z}_t$ in~\eqref{eq:ddpm}, by substituting $\mathbf{A}_t$ instead of $\boldsymbol{C}_t$ in~\eqref{eq:std-norm}. Encoding and decoding algorithms are provided in App.~\ref{app:algorithms}.

The encoder requires one FWHT and top-$M$ selection; the decoder requires one FWHT to reconstruct the selected noise. Both run in $O(d\log d)$ time and $O(d)$ memory without storing a dictionary, compared with the $O(Kd)$ Gaussian codebook of Turbo-DDCM (App.~\ref{app:codebook_cost}).
Note that although residual-dependent selection means the resulting noise is not itself Gaussian, the same is true of Turbo-DDCM's selected Gaussian combinations -- Sec.~\ref{sec:experiments} evaluates reconstruction quality directly.

\vspace{-2pt}
\subsection{Bitstream optimization}
\vspace{-2pt}
\label{sec:method:rate}
Turbo-DDCM encodes the $M$ selected indices through their lexicographic rank, requiring $\lceil \log_2 \tbinom{K}{M} \rceil$ bits for the selected index set. However, its implementation relies on a large precomputed table of binomial coefficients whose memory and construction costs scale poorly with dictionary size $K$, becoming impractical in high-dimensional regimes (we refer to App.~\ref{app:bitstream_protocol} for further discussion). Instead, we propose a table-free lexicographic coding scheme that is tractable at large dimensions, while preserving the same combinatorial coding principle.

We propose two table-free mechanisms (App.~\ref{app:bitstream_protocol}): $(i)$~an incremental recurrence advances the required binomial coefficients using exact Pascal-rule updates, avoiding their repeated computation and storage, and $(ii)$~a hockey-stick-identity~\citep{jones1996generalized} formulation expresses each selected index's contribution using directly evaluated binomial coefficients and unranks via binary search. The former is effective for dense selections, while the latter is particularly efficient when $M \ll d$.

Furthermore, for a dictionary of size $K$, we divide the candidate atoms into $G$ groups of size $K_G := K/G$ and encode each group independently. We first transmit the count vector $(m_1, \ldots, m_G)$, which indicates the number of atoms in each group ($\sum_{g=1}^G m_g = M$), using a combinatorial stars-and-bars code~\citep{feller1950introduction}. This code uses $\log_2 \lceil \tbinom{M+G-1}{G-1} \rceil$ bits to describe which of the $\tbinom{M+G-1}{G-1}$ possible allocations has been used. This is followed by the local lexicographic rank and coefficient bits within each group. This trades a small rate overhead for $G$ smaller combinatorial operations that can be evaluated independently and in parallel. The resulting bitrate is
\begin{equation}
\label{eq:bpp}
\mathrm{BPP} = \frac{S\left(L_{\mathrm{counts}} + \sum_{g=1}^{G}\left(\left\lceil \log_2 \tbinom{K_G}{m_g} \right\rceil + m_g C\right)\right)}{\text{number of pixels}}, \qquad L_{\mathrm{counts}} = \left\lceil\log_2\tbinom{M+G-1}{G-1}\right\rceil,
\end{equation}
where $S$ is the number of bit-carrying sampling steps, which is $T-N-1$ for DDPM and $T$ for consistency-models (see App.~\ref{app:sampling_formulations}). The single Hadamard dictionary has $K = d$ atoms, giving $K_G = d/G$. Setting $G = 1$ recovers a single flat lexicographic rank.

\vspace{-2pt}
\subsection{Extension: overcomplete Hadamard dictionary}
\vspace{-2pt}
\label{sec:method:overcomplete}

At small $M$, \ie in the low-bitrate regime, reconstruction quality becomes particularly sensitive to the alignment between $\mathbf{r}_t$ and the available basis directions. We therefore extend the single Hadamard basis into an overcomplete dictionary by drawing $B$ independent Rademacher diagonals $\mathbf{D}_{t,1},\ldots,\mathbf{D}_{t,B}$ from the shared seed. This gives $B$ independently randomized orthogonal bases $\mathbf{A}_{t,b}:=\mathbf{D}_{t,b}\mathbf{H}$, which we concatenate as
\begin{equation}
\label{eq:overcomplete-dictionary}
\mathbf{A}_t^{\mathrm{oc}} := [\mathbf{A}_{t,1} \;\cdots\; \mathbf{A}_{t,B}] \in \mathbb{R}^{d \times Bd}.
\end{equation}
The encoder can therefore select from $Bd$ candidate directions rather than $d$ of a single basis.

Each $\mathbf{A}_{t,b}$ is individually orthogonal by~\eqref{eq:DH-orthonormal}, but atoms belonging to different randomized bases are not mutually orthogonal to each other: $\mathbf{A}_{t,b}^\top\mathbf{A}_{t,b'}$ is a dense matrix for $b\neq b'$, not the identity. Global top-$M$ thresholding is therefore no longer an exact solution of~\eqref{eq:our-problem}, however we still use this approach, as in Turbo-DDCM. 
Each basis contributes its own scores $\mathbf{u}_{t,b} := \mathbf{A}_{t,b}^\top\mathbf{r}_t$, at the cost of one FWHT, and the encoder selects the $M$ largest-magnitude coefficients across all $B$ score vectors. On the synthesis side, once the support is selected, the least-squares coefficients on the support can be recovered using the pseudo-inverse of the sub-matrix corresponding to the chosen atoms. In practice, however, we find that due to the coefficient quantization, using the analysis coefficients instead of the least-squares ones leads to an almost identical noise direction (the comparison is provided in App.~\ref{app:trans_pinv}).

An index now identifies one of $Bd$ atoms, which we encode using the grouped protocol of Sec.~\ref{sec:method:rate} by treating each randomized basis as an individual group ($G = B$). Because each of the $B$ bases contains $d$ atoms, the per-group size matches the basis dimension ($K_G = d$). The richer dictionary therefore incurs additional transform and index costs; at a fixed bit budget the extension must select fewer atoms than \oursshort{}. Appendix~\ref{app:overcomplete_hadamard} gives the construction in full and an exact and efficient way to compute the global top-$M$.

%% file: sections/experiments.tex
\vspace{-2pt}
\section{Application to image compression}
\vspace{-2pt}
\label{sec:experiments}
At high latent dimensions, Turbo-DDCM faces a practical trade-off -- a complete codebook provides the full set of candidate directions but incurs large memory and computational costs, while capping the codebook reduces these costs at the expense of fidelity due to restricted subspace span. Here, we demonstrate that \oursshort{} avoids this trade-off at FLUX scale ($d=2^{18}$), matching the quality of the complete codebook while retaining the efficiency of the capped variant.

\myparagraph{Implementation details.} We evaluate two backbones: Stable Diffusion (SD)~2.1~\citep{rombach2022high} at $512\times512$ resolution, where the latent dimension is $d = 2^{14}$, and FLUX at $1024\times1024$ resolution, where $d = 2^{18}$. For both, we run $T=30$ steps, use $C=1$ bit per coefficient, rely on the incremental recurrence (Pascal-rule) implementation for bitstream coding, and sweep $M$ to trace out the rate-distortion-perception tradeoff~\citep{blau2019rethinking}. We use $G=1$ groups for SD~2.1 and $G=8$ for FLUX. Appendix~\ref{app:image_experiment_configurations} lists additional implementation details and the hyperparameters for each method.

\myparagraph{Evaluation.} For SD~2.1 we report results on the Kodak24~\citep{kodak1999} dataset and additionally evaluate on the DIV2K~\citep{agustsson2017ntire} and CLIC2020~\citep{CLIC2020} datasets in App.~\ref{app:additional_results}; all three datasets are resized and center-cropped to $512\times512$. For FLUX we evaluate on DIV2K, resized and center-cropped to $1024\times1024$. We report round-trip (compression and decompression) time for every method on a single NVIDIA A6000 GPU. Distortion is measured by PSNR and perceptual quality by FID~\citep{heusel2017gans}, computed on $64\times64$ patches following \citet{mentzer2020high}. For training-free diffusion-based methods, we further report the excess memory required for dictionary storage, as well as the disk storage required by Turbo-DDCM's precomputed binomial tables. Appendix~\ref{app:additional_results} additionally reports SSIM~\citep{wang2004image}, LPIPS~\citep{zhang2018unreasonable}, and separate encoding and decoding times.

\begin{figure}[t]
    \vspace{-.5em}
    \centering
    \begin{subfigure}{\linewidth}
        \centering
        \includegraphics[width=\linewidth]{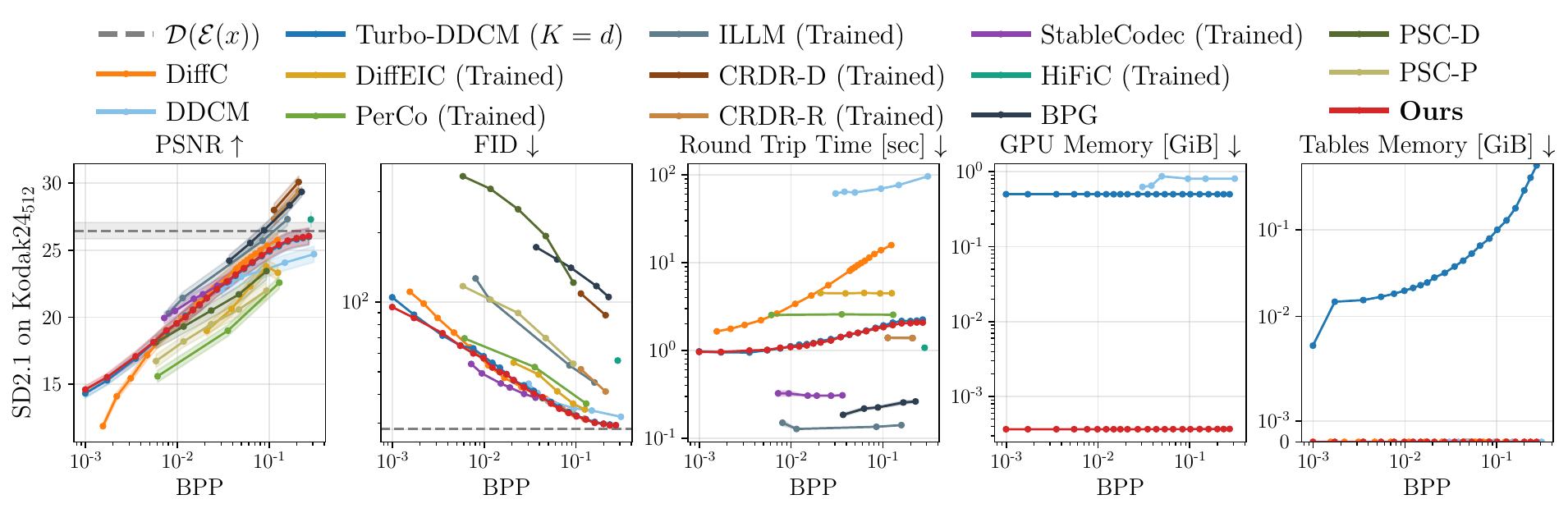}
    \end{subfigure}
    
    \vspace{-.5em}
    
    \begin{subfigure}{\linewidth}
        \centering
        \includegraphics[width=\linewidth]{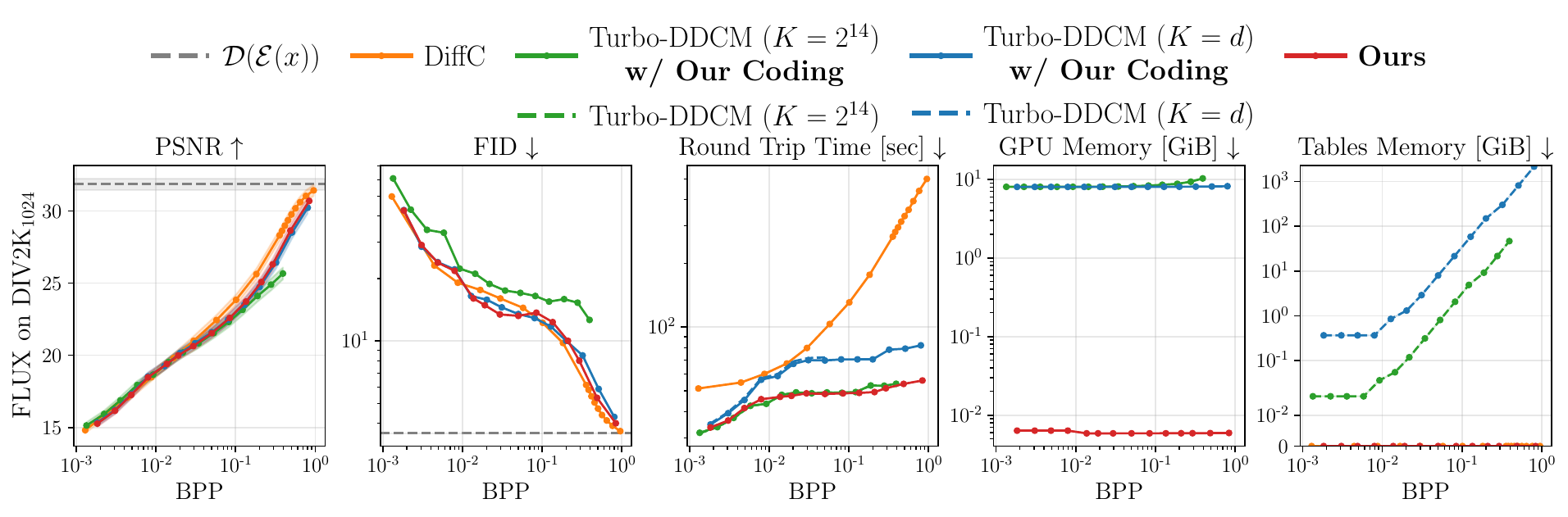}
    \end{subfigure}
    \vspace{-1.5em}
    \caption{\textbf{Quantitative comparison of image compression methods.} We compare \oursshort{} with training-free diffusion-based and other compression methods in terms of rate, distortion, and perception. Distortion is measured by PSNR, perceptual quality by FID (presented in log-scale), and runtime by round-trip encoding and decoding time. GPU memory refers to the excess memory used by training-free methods for dictionary storage, while tables memory refers to disk storage required by Turbo-DDCM's precomputed combinatorial tables. The first row corresponds to training-free methods utilizing the SD~2.1 backbone ($512 \times 512$, $d = 2^{14}$) and other compression methods. The second row corresponds to FLUX backbone ($1024 \times 1024$, $d = 2^{18}$). The dashed gray vertical line denotes the VAE encoder-decoder bound. \oursshort{} achieves competitive results on the rate-distortion-perception tradeoff, all the while being the fastest training-free method, avoiding explicit codebook materialization. Error bars represent the standard error of the mean (SEM).}
    \label{fig:main_images_quantitative_comparisons}
    \vspace{-.5em}
\end{figure}

\begin{figure}[t]
    \centering
    \begin{subfigure}{\linewidth}
        \centering
        \includegraphics[width=\linewidth]{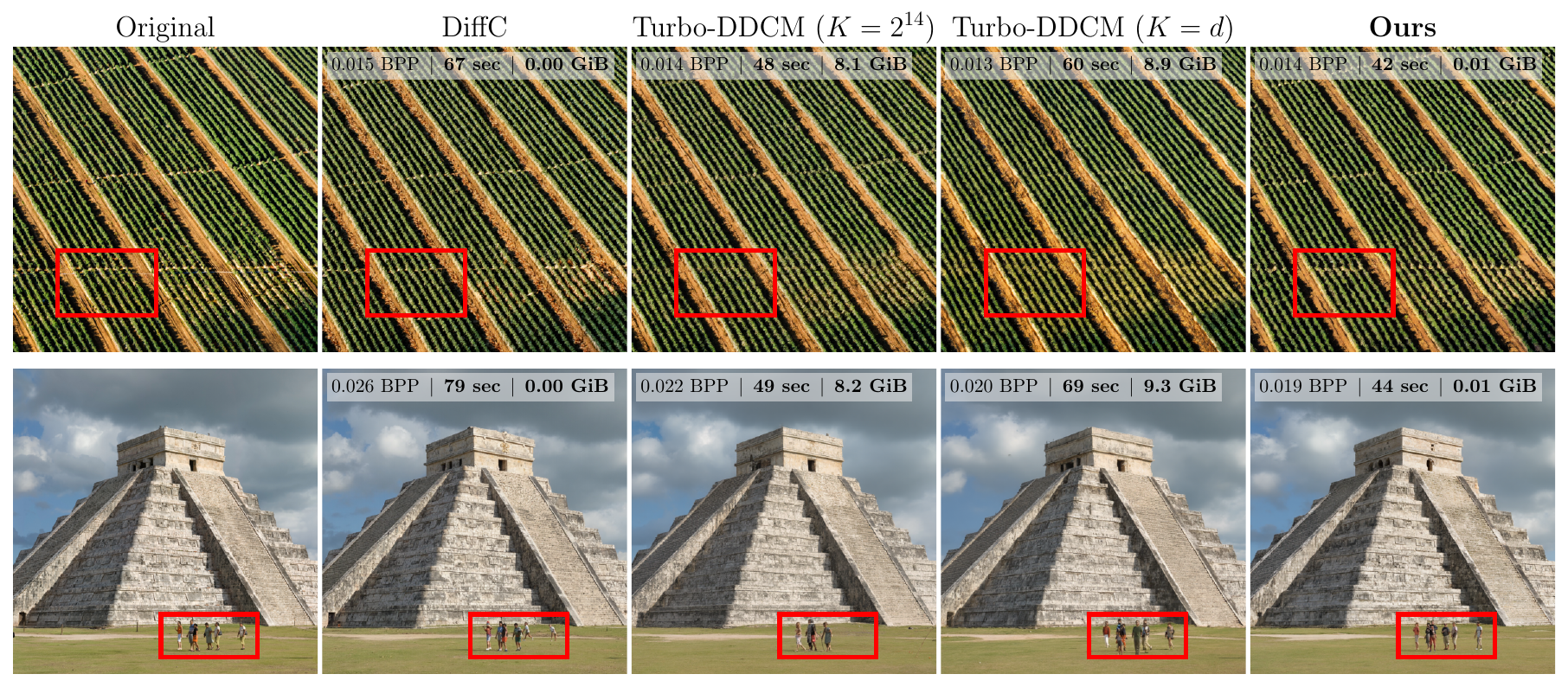}
    \end{subfigure}
    \vspace{-1em}
    \caption{\textbf{Qualitative comparison of training-free diffusion-based image compression methods (FLUX on DIV2K).} \oursshort{} is the fastest and preserves visual quality and fidelity to the original image, while avoiding the memory overhead of an explicit codebook.}
    \label{fig:main_images_qualitative_comparisons}
    \vspace{-.5em}
\end{figure}

\myparagraph{Competing methods.} We compare against every training-free diffusion-based codec that shares our exact backbone and requires no training: DiffC~\citep{theis2022lossy,vonderfecht2025lossy}, DDCM~\citep{ohayon2025compressed}, Turbo-DDCM~\citep{vaisman2026turbo} utilizing a capped codebook ($K=2^{14}$) as well as a complete one ($K=d$), and PSC~\citep{elata2025psc} in its distortion- and perception-oriented variants (PSC-D, PSC-P). We additionally compare against codecs trained or fine-tuned specifically for compression -- DiffEIC~\citep{li2024toward}, PerCo~\citep{korber2024perco}, ILLM~\citep{muckley2023improving}, CRDR~\citep{iwai2024controlling} (its D and R variants), StableCodec~\citep{zhang2025stablecodec}, and HiFiC~\citep{mentzer2020high}. We also compare against the classical codec BPG~\citep{bellard2018bpg}. For FLUX we compare to training-free methods only: DiffC and Turbo-DDCM using both codebook sizes. Note that PSC is omitted from the runtime comparisons (Figs.~\ref{fig:main_images_quantitative_comparisons} and~\ref{fig:supp_images_quantitative_comparisons_complementary}) due to its slow runtime.

\myparagraph{Results.}
Figure~\ref{fig:main_images_quantitative_comparisons} shows that \oursshort{} closely matches complete-codebook Turbo-DDCM ($K=d$) and DiffC in terms of PSNR and FID across the evaluated bitrate regime. On FLUX, capping Turbo-DDCM's codebook below the latent dimension $d$ reduces quality, while using the complete codebook substantially increases memory and runtime. At matched bitrate on SD~2.1, \oursshort{} and Turbo-DDCM ($K=d$) both achieve a $1.3$\,s round trip time, compared with $5.5$\,s for DiffC and $61$\,s for DDCM. At FLUX scale, \oursshort{} reduces the round trip time from $66$\,s for Turbo-DDCM ($K=d$) to $44$\,s, while avoiding explicit codebook materialization.
These results show that the structured dictionary retains the quality benefit of a complete codebook while scaling more favorably in memory and computation. Figure~\ref{fig:main_images_qualitative_comparisons} provides corresponding FLUX reconstructions, including the round trip time and the total required memory for each method (GPU and tables memory), and App.~\ref{app:additional_results} reports additional datasets and metrics.

%% file: sections/application_to_streaming.tex
\vspace{-2pt}
\section{Application to streaming video compression}
\vspace{-2pt}
\label{sec:vid_streaming_app}
Beyond the high-dimensional setting considered for images, causal AR video compression introduces two additional challenges: codec throughput must keep pace with generation, and compression must continue beyond the backbone's finite training context without losing temporal consistency. We adapt our framework to the CF AR video backbone and address the latter challenge with a re-encoded-tail KV-cache refresh, enabling long-form, close to real-time video compression while retaining context across segment boundaries.

\myparagraph{Implementation details.}
We evaluate our method using CF at $832\times 480$ resolution. We apply the residual-coding procedure to each causal latent chunk, using CF's consistency trajectory in place of the DDPM trajectory used for images. We run $T=4$ consistency steps and use the $B=8$ overcomplete (OC) dictionary (Sec.~\ref{sec:method:overcomplete}). Unlike DDPM, in the consistency formulation we choose to use no deterministic tail ($N=0$), so all $T$ steps perform noise selection and carry bits. We further encode the initial-step noise using a proxy residual, instead of a random noise. We refer to App.~\ref{app:sampling_formulations} for additional details about the consistency formulation and the initial-step proxy residual. We use our efficient hockey-stick-identity lexicographic implementation for all applicable CF-based baselines to isolate codec overhead.

\myparagraph{Re-encoded-tail KV-cache refresh.}
CF is trained on contexts of $21$ latent frames ($81$ pixel frames), corresponding to $7$ chunks of $3$ latent frames each. Extending the AR cache beyond this regime leads to severe degradation, while simply resetting the cache at each continuation boundary discards temporal context and introduces discontinuities (see ablation in App.~\ref{app:additional_results}). We therefore refresh the cache at every segment boundary by re-encoding the last $9$ reconstructed pixel frames into a fresh KV-cache prefix. Since these reconstructed frames are available identically to both encoder and decoder, the refreshed prefix requires no additional transmitted bits. The next segment is then compressed conditioned on this prefix, preserving recent temporal context while keeping the cache within the model's training regime. The first newly coded chunk after a refresh is particularly difficult because it has the least generated context available. We therefore use a larger selection budget $M_{\mathrm{start}}\geq M$ for this chunk, after which compression returns to the budget of $M$ indices per step. Appendix~\ref{app:video_experiment_configurations} provides additional details of the re-encoded-tail mechanism, $M_{\mathrm{start}}$ values, and lists the hyperparameters for each method.

\begin{figure}[t]
    \vspace{-.5em}
    \centering
    \begin{subfigure}{\linewidth}
        \centering
        \includegraphics[width=\linewidth]{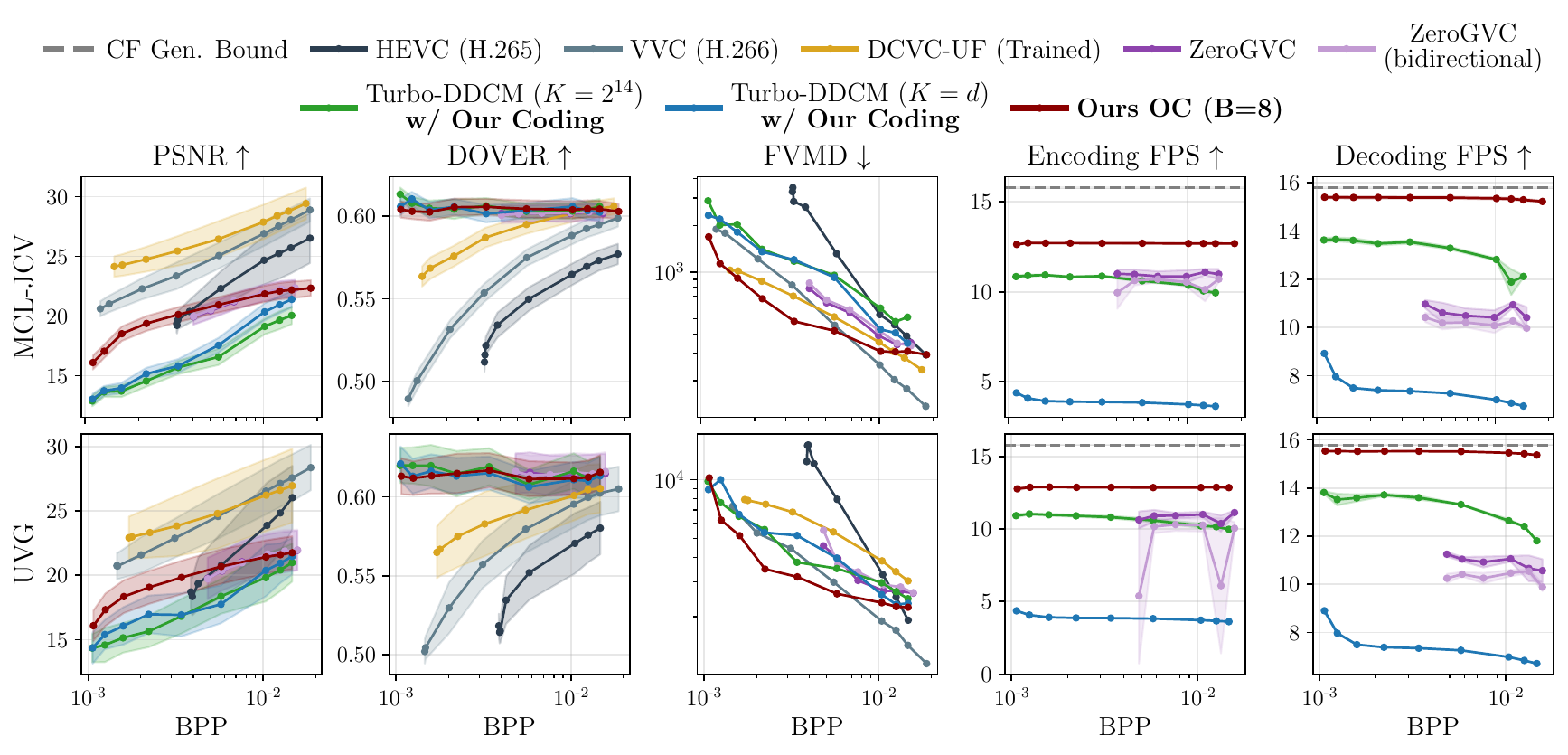}
    \end{subfigure}
    \vspace{-1.5em}
    \caption{\textbf{Quantitative comparison of video compression methods.} We compare \oursshort{} OC ($B=8$) with training-free methods using the same CF backbone and other compression baselines in terms of rate, distortion and perception. Distortion is measured by PSNR, and perceptual quality is measured by DOVER and FVMD (presented in log-scale). Encoding and decoding FPS report compression and decompression throughput, and are compared among methods using the CF backbone.
    \oursshort{} achieves competitive rate-distortion-perception performance among the CF-based training-free methods while providing the highest measured encoding and decoding throughput, without codebooks materialization. Error bars represent the standard error of the mean (SEM).}
    \label{fig:main_video_quantitative_comparisons}
    \vspace{-.5em}
\end{figure}

\begin{figure}[t]
    \centering
    \begin{subfigure}{\linewidth}
        \centering
        \includegraphics[width=\linewidth]{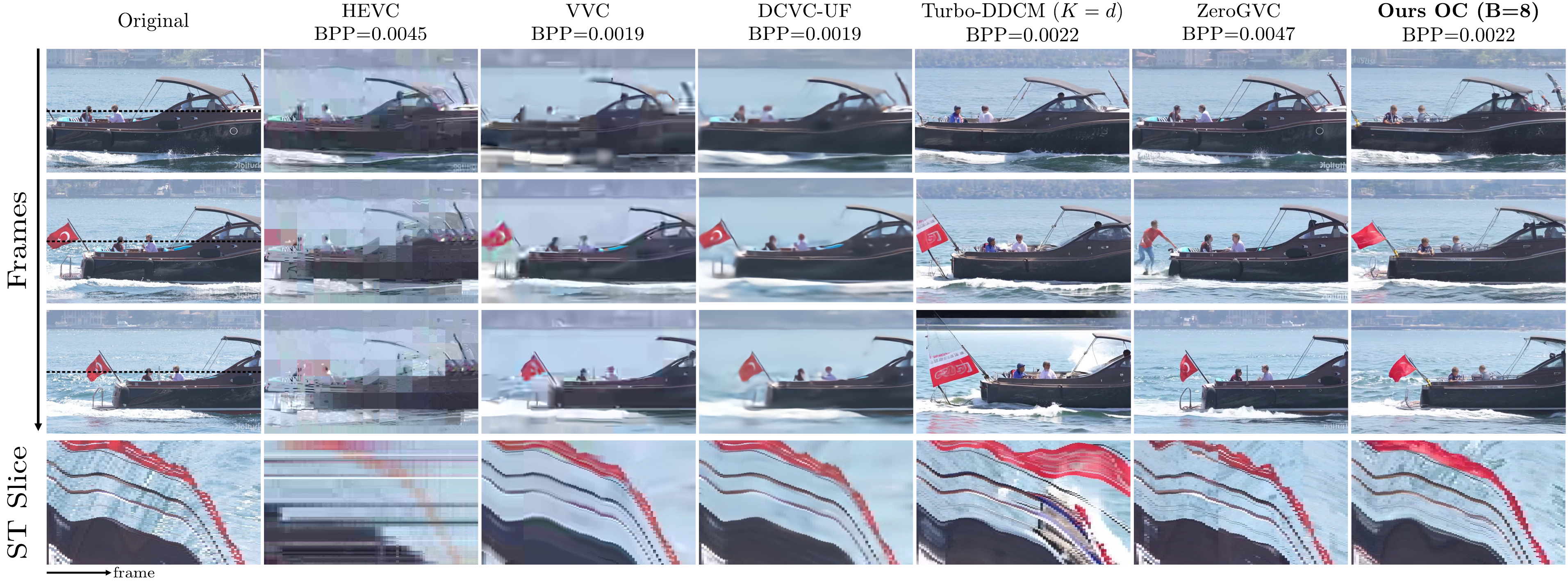}
    \end{subfigure}
    \vspace{-1.5em}
    \caption{\textbf{Qualitative comparison of video-compression methods.} The example is taken from the UVG dataset at $832\times480$ resolution. Our method preserves source content and temporal consistency while providing the highest measured throughput among the evaluated CF-based methods. The space-time (ST) slice, which shows the temporal evolution of the pixels marked by the dashed line, reveals temporal discontinuities at ZeroGVC’s I-frame boundaries.}
    \label{fig:main_video_qualitative_comparisons}
    \vspace{-.5em}
\end{figure}

\myparagraph{Evaluation.} We evaluate on UVG~\citep{mercat2020uvg} and MCL-JCV~\citep{wang2016mcl}, resized and center-cropped to $832\times 480$, and in App.~\ref{app:additional_results} we further evaluate HEVC Class B/C/D~\citep{bossen2010common}. We measure distortion by PSNR, and perceptual quality with DOVER~\citep{wu2023exploring} and FVMD~\citep{liu2024frchet}; additional metrics are reported in App.~\ref{app:additional_results}. Encoding and decoding throughput are measured on a single NVIDIA H200 GPU and compared directly among methods leveraging the same CF backbone.

\myparagraph{Competing methods.} We compare against the classical codecs HEVC~\citep{sullivan2012overview} and VVC~\citep{bross2021overview}, the recent neural codec DCVC-UF~\citep{li2026ultra}, and training-free generative compression methods using the same CF backbone: ZeroGVC~\citep{gao2026zerogvc}, and Turbo-DDCM with both a capped codebook ($K=2^{14}$) and a complete one ($K=d$). As no official CF implementation is available for ZeroGVC or Turbo-DDCM, we implement these baselines following their respective formulations.

\myparagraph{Results.} 
Figure~\ref{fig:main_video_quantitative_comparisons} presents the quantitative results. HEVC, VVC, and DCVC-UF achieve higher PSNR but substantially weaker perceptual quality as measured by DOVER and FVMD. Among the training-free methods sharing the CF backbone, ZeroGVC achieves comparable PSNR but weaker temporal and motion consistency as captured by FVMD, together with lower throughput. Both Turbo-DDCM variants achieve lower PSNR and perceptual quality, while the complete-codebook variant additionally incurs substantially lower throughput. In contrast, \oursshort{} provides competitive rate-distortion-perception performance while achieving the highest measured encoding and decoding throughput among the evaluated CF-based methods. Figure~\ref{fig:main_video_qualitative_comparisons} illustrates the corresponding visual and temporal behavior, including discontinuities at segment boundaries in ZeroGVC. Additional quantitative and qualitative results are provided in App.~\ref{app:additional_results}.

%% file: sections/conclusion.tex
\vspace{-2pt}
\section{Conclusion}
\vspace{-2pt}

We introduced \oursshort{}, a training-free generative codec that replaces Turbo-DDCM's dense Gaussian codebooks with implicit randomized Hadamard dictionaries and its precomputed lexicographic table with scalable table-free, group-factorized coding. This enables efficient compression at large latent dimensions: on FLUX, \oursshort{} matches the fidelity of a complete Turbo-DDCM codebook without materializing it, while retaining the runtime efficiency of a capped codebook. For video, we combine an overcomplete Hadamard dictionary with a re-encoded-tail KV-cache refresh to adapt the framework to the CF AR backbone, enabling long-form, close to real-time streaming compression. \oursshort{} achieves competitive rate-distortion-perception performance while providing the highest measured throughput among the evaluated methods sharing the CF backbone.

%% file: sections/ai_use.tex
\subsection*{AI use statement}
Generative AI tools were used in this paper for minor text polishing to improve readability, and for editing software code.
We have not used AI to: generate synthetic data sets, help develop theoretical models or conceptual frameworks, formulate mathematical claims, provide critical ingredients for proving mathematical claims, assist in the writing of proofs, propose or refine hypotheses, design or provide feedback on research methodology or experiments, implement methods from scratch, assist with translation, clean and reformat dataset, support qualitative and thematic data analysis, and interpret results.
Additionally, we used generative AI tools to create and modify figures, to improve the style.
We have reviewed all AI-assisted work. We take responsibility for the final content of this work, including text, claims or artifacts produced with the aid of generative AI.

%% file: sections/ethics_reproducibility.tex
\subsection*{Ethics statement}
This work builds upon pre-trained generative models, and thus inherits the broader ethical considerations associated with their use. Such models may reflect or amplify societal biases present in the training data, and their outputs could be misinterpreted or misused in sensitive applications. In addition, our approach involves large-scale diffusion and consistency models, which carry the potential risk of being repurposed for harmful or malicious purposes. We emphasize that our contributions are intended solely for advancing research in generative modeling.

\subsection*{Reproducibility statement}
Our results can be fully reproduced by implementing the method as described in Sec.~\ref{sec:method}, and summarized as pseudo-code in App.~\ref{app:algorithms}, with the hyperparameters as provided in Apps.~\ref{app:image_experiment_configurations} and~\ref{app:video_experiment_configurations}, and the bitstream implementation details as provided in App.~\ref{app:bitstream_protocol}.

\subsection*{Acknowledgments}
The Miriam and Aaron Gutwirth Memorial Fellowship supported the research of OR and VK.
This research was supported by the Israel Science Foundation (grant no. 2318/22) and by the Joseph and Sadie Danciger Chair in Engineering.

%% file: supplementary/supplementary.tex
\beginsupplement
\appendix
\mySuppTitle
\input{supplementary/experiment_configurations}

\input{supplementary/experiment_configurations_video}
\input{supplementary/additional_results}
\input{supplementary/sampling_formulations}
\input{supplementary/bitstream_protocol}
\input{supplementary/codebook_cost}
\input{supplementary/overcomplete_hadamard}
\input{supplementary/algorithms}

%% file: supplementary/experiment_configurations.tex
\section{Experiment image compression configurations}
\label{app:image_experiment_configurations}

We report the exact configurations used for the image experiments in Figs.~\ref{fig:main_images_quantitative_comparisons},~\ref{fig:supp_images_quantitative_comparisons}, and~\ref{fig:supp_images_quantitative_comparisons_complementary}. Training-free methods are evaluated with the same pretrained backbones reported in the main paper, while trained baselines use their publicly released checkpoints without additional training or fine-tuning.

\subsection{SD~2.1 (Kodak24, DIV2K, CLIC2020)}

All three datasets use the same method set and the same per-method hyperparameters.

\myparagraph{Training-free methods.}
\begin{itemize}
  \item \textbf{\oursshort{}}: $T=30$ (latent dimension $d=2^{14}$), $M$ ranging from $22$ to $300$ ($22$ points); the DDIM tail length $N + 1$ follows Turbo-DDCM's heuristic for setting it automatically per $M$.
  \item \textbf{Turbo-DDCM ($K=d=2^{14}$)}: $T=30$, using the same $22$-point $M$-grid and per-$M$ deterministic-tail schedule as \oursshort{}. We compute lexicographic rank and unrank operations using our table-free recurrence (App.~\ref{app:bitstream_protocol}) rather than Turbo-DDCM's original precomputed table, so runtime comparisons isolate residual-selection/codebook cost rather than avoidable table-construction overhead.
  \item \textbf{DDCM}: hyperparameters from the original paper -- $T=1000$ for $K\in\{256,1024,8192\}$ with $M=1$, and $T=1000$ for $K=2048$ with $M\in\{2,3,6\}$ and $C=3$.
  \item \textbf{DiffC}: the schedulers, $D_{\mathrm{KL}}$ schedule, and reconstruction timesteps suggested by the original authors, unchanged from their released configuration.
  \item \textbf{PSC-D / PSC-P}: rank $12$ and $2^i, i\in\{4,5,6,7,8\}$ iterations; PSC-D averages $64$ posterior samples per reconstruction, and PSC-P uses the authors' own nested-posterior-sampling estimator.
\end{itemize}

\myparagraph{Trained baselines.}
\begin{itemize}
  \item \textbf{DiffEIC}: the released \texttt{1\_2\_X} checkpoints for $\texttt{X}\in\{1,2,4,8,16\}$.
  \item \textbf{PerCo (SD)}: three publicly available checkpoints, with the default hyperparameters.
  \item \textbf{StableCodec}: the released checkpoints \texttt{stablecodec\_ftX} for $\texttt{X} \in \{2,4,8,12,24,32\}$.
  \item \textbf{HiFiC}: the official checkpoint for the low-quality regime, run from a PyTorch reimplementation rather than the official TensorFlow release.
  \item \textbf{ILLM}: the released \texttt{msillm\_quality\_X} for $\texttt{X}\in\{2,3\}$ and \texttt{msillm\_quality\_vloY} for $\texttt{Y}\in\{1,2\}$ checkpoints.
  \item \textbf{CRDR-D / CRDR-R}: QF $qf\in\{0,1,2,3,4\}$, with $\beta\in\{0, 3.84\}$ for CRDR-D and CRDR-R.
\end{itemize}

\myparagraph{Classical codec.}
\begin{itemize}
  \item \textbf{BPG}: QP values $q \in \{51,48,46,42,40\}$.
\end{itemize}

\subsection{FLUX (DIV2K)}

Four methods are evaluated at FLUX scale ($d=2^{18}$): \oursshort{}, Turbo-DDCM at $K=2^{14}$ and at $K=d$, and DiffC.

\begin{itemize}
  \item \textbf{\oursshort{}}: $T=30$, $M$ ranging from $43$ to $4641$ ($14$ rate points). We use the same step scheduler as SD~2.1, following \citet{vonderfecht2025lossy}, but subsampled non-uniformly across the sweep: at low bitrate the $30$ steps crowd toward the trajectory's higher-noise timesteps, moving to an evenly-spaced subsampling only once the bitrate is high enough. The DDIM tail length $N$ follows the same principle, shrinking from a larger value at low bitrate toward zero as $M$ grows.
  \item \textbf{Turbo-DDCM ($K=2^{14}$)}: the dense-codebook baseline capped well below $d$, with $N$ matched per rate point to hit the same $14$ bpp targets as \oursshort{}.
  \item \textbf{Turbo-DDCM ($K=d=2^{18}$)}: we use the same rate targets as the capped-codebook variant. Materializing the full $d\times K$ Gaussian codebook would require approximately $128$\,GiB in FP16 (App.~\ref{app:codebook_cost}), so we generate and process it in blocks of $2^{14}$ columns.
  \item \textbf{DiffC}: the schedulers, $D_{\mathrm{KL}}$ schedule, and reconstruction timesteps suggested by the original authors, unchanged from their released FLUX configuration.
\end{itemize}

%% file: supplementary/experiment_configurations_video.tex
\section{Experiment video compression configurations}
\label{app:video_experiment_configurations}

We report the configurations used for all video experiments in Sec.~\ref{sec:vid_streaming_app}, including the re-encoded-tail KV-cache geometry and the rate-matched configurations for each training-free method. The same configurations are used across all evaluated video datasets.

\myparagraph{Re-encoded-tail geometry.}
Each CF chunk contains $3$ latent frames with flattened dimension $d=299{,}520$. For the Hadamard transform we zero-pad each chunk to $d=2^{19}$ and use $T=4$ bit-carrying consistency steps. The first video segment generates $21$ latent frames ($81$ pixel frames). Each continuation generates $18$ new latent frames, preceded by a $3$ latent-frame prefix obtained by re-encoding the last $9$ reconstructed pixel frames from the preceding segment. Because these frames are already available at both encoder and decoder, the refreshed prefix requires no transmitted bits.

At the first newly coded chunk after each refresh, we use a larger budget $M_{\mathrm{start}}\ge M$ to compensate for the reduced generated context. We choose $M_{\mathrm{start}}$ per method such that this chunk carries at least $20{,}000$ bits across its four selection steps: $M_{\mathrm{start}}=390$ for \oursshort{}, $308$ for \oursshort{} OC, $724$ for Turbo-DDCM ($K=2^{14}$), and $420$ for Turbo-DDCM ($K=d$). Once the steady-state $M$ reaches the corresponding floor, we set $M_{\mathrm{start}}=M$.

\myparagraph{Training-free methods.}
All training-free generative methods are evaluated using the same CF backbone at its native $832\times480$ resolution.
\begin{itemize}
  \item \textbf{\oursshort{}}: $10$ reference bitrates with $M$ ranging from $16$ to $2048$ over the padded support $d=2^{19}$. We use hockey-stick ranking and binary-search unranking (App.~\ref{app:bitstream_protocol}).
  \item \textbf{Overcomplete dictionary \oursshort{}} ($B=8$): $10$ configurations rate-matched to the corresponding single-basis \oursshort{} points. The OC selection budgets are smaller than the reference $M$ values because each selected atom is drawn from the larger $Bd$-atom dictionary; the exact matched pairs are listed in App.~\ref{app:overcomplete_hadamard}.
  \item \textbf{Turbo-DDCM ($K=2^{14}$)}: configurations are rate-matched to the same target bit budgets using a deterministic Gaussian codebook; search and reconstruction use FP16.
  \item \textbf{Turbo-DDCM ($K=d=299{,}520$)}: configurations are rate-matched in the same way using the complete Gaussian codebook. Because the complete codebook cannot be stored alongside the backbone, it is processed in blocks of $2^{14}$ columns.
  \item \textbf{ZeroGVC}: we use independent $32$-pixel-frame GOPs. Each GOP begins with a OneDC~\citep{xue2026one} ($\lambda=1$) I-frame that resets the autoregressive cache; subsequent latents use framewise CF with a deterministic $K=2^{14}$ Turbo-DDCM codebook and $4$ denoising steps. We sweep $M$ from $64$ to $1024$ over $6$ configurations and evaluate both bidirectional and non-bidirectional streaming variants. GOPs are repeat-padded to a valid $1+4q$ length and cropped back after decoding.
\end{itemize}

\myparagraph{Trained baselines.}
\begin{itemize}
  \item \textbf{DCVC-UF}: \href{https://github.com/microsoft/dcvc}{Microsoft's official implementation} and checkpoints, with $\mathrm{QP\text{-}I}=\mathrm{QP\text{-}P}\in\{0,\dots,63\}$ found via a two-pass search targeting each rate point.
\end{itemize}

\myparagraph{Classical codecs.}
\begin{itemize}
  \item \textbf{HEVC}: FFmpeg's \texttt{libx265} implementation, with two-pass average-bitrate encoding targeting each rate point.
  \item \textbf{VVC}: \href{https://github.com/fraunhoferhhi/vvenc}{Fraunhofer's VVenC/VVdeC implementation}, with the same two-pass average-bitrate targeting.
\end{itemize}

\clearpage

%% file: supplementary/additional_results.tex
\section{Additional results}
\label{app:additional_results}

\subsection{Image compression quantitative comparison}
Figure~\ref{fig:supp_images_quantitative_comparisons} extends the SD~2.1 comparison of Fig.~\ref{fig:main_images_quantitative_comparisons} to DIV2K and CLIC2020. Across both datasets, \oursshort{} follows the same trend observed on Kodak24: it achieves competitive rate-distortion-perception performance while avoiding explicit Gaussian-codebook storage and retaining low round-trip time among the evaluated training-free methods.

\begin{figure}[htbp]
    \centering
    \begin{subfigure}{\linewidth}
        \centering
        \includegraphics[width=\linewidth]{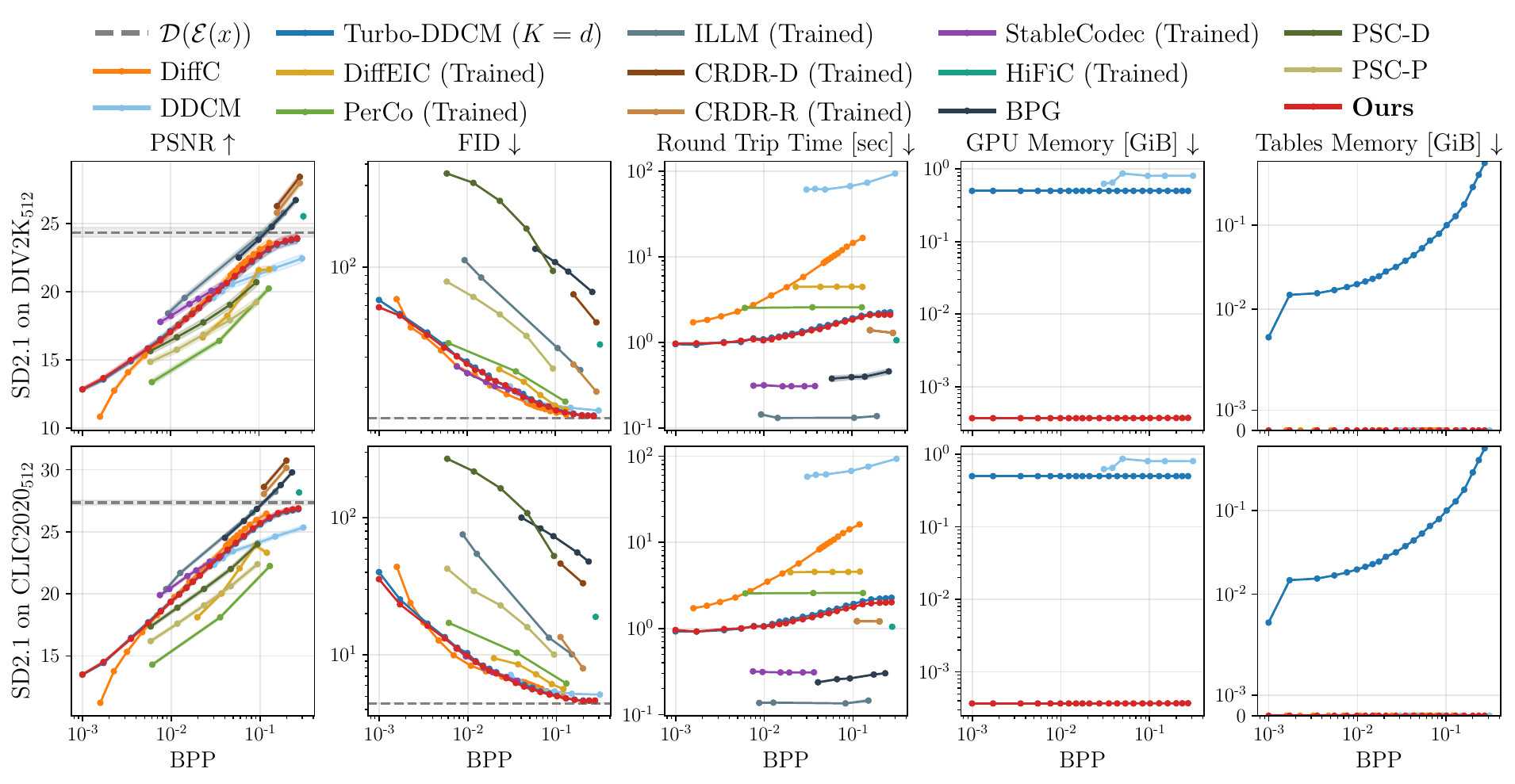}
    \end{subfigure}
    \vspace{-1.5em}
    \caption{\textbf{Additional quantitative comparison of image compression methods (SD2.1).} We compare \oursshort{} with training-free diffusion-based and other compression methods in terms of rate, distortion and perception. Distortion is measured by PSNR, perceptual quality is measured by FID (presented in log-scale), and runtime is measured by round-trip encoding and decoding time. GPU memory refers to the excess memory used by training-free methods for dictionary storage, while tables memory denotes disk storage required by Turbo-DDCM's precomputed combinatorial tables. The first row corresponds to DIV2K dataset, and the second row corresponds to CLIC2020 dataset. Both datasets are resized and center-cropped to resolution of $512 \times 512$, corresponding to $d = 2^{14}$. The dashed gray vertical line corresponds to the encoder-decoder bound. \oursshort{} achieves competitive results on the rate-distortion-perception tradeoff, all the while being the fastest training-free method, avoiding explicit codebook materialization. Error bars represent the standard error of the mean (SEM).}
    \label{fig:supp_images_quantitative_comparisons}
\end{figure}

Figure~\ref{fig:supp_images_quantitative_comparisons_complementary} reports complementary metrics for the image experiments in Figs.~\ref{fig:main_images_quantitative_comparisons} and~\ref{fig:supp_images_quantitative_comparisons}. We include SSIM and LPIPS together with separate encoding, decoding, and round-trip times and the memory requirements of the training-free methods. PSC is omitted from the timings and memory consumption due to its slow runtime.

\begin{figure}[htbp]
    \centering
    \begin{subfigure}{\linewidth}
        \centering
        \includegraphics[width=\linewidth]{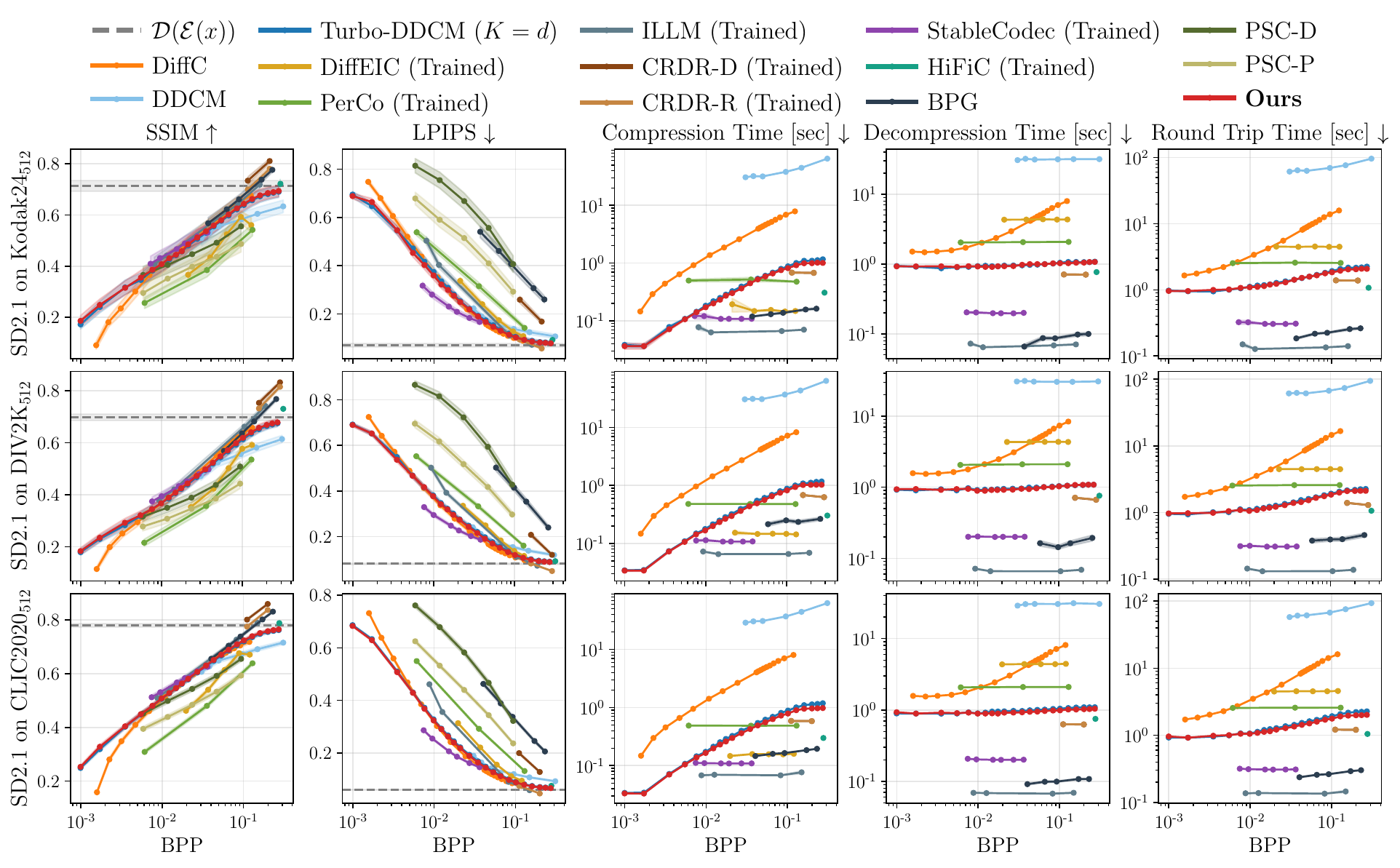}
    \end{subfigure}
    
    \vspace{-.5em}
    
    \begin{subfigure}{\linewidth}
        \centering
        \includegraphics[width=\linewidth]{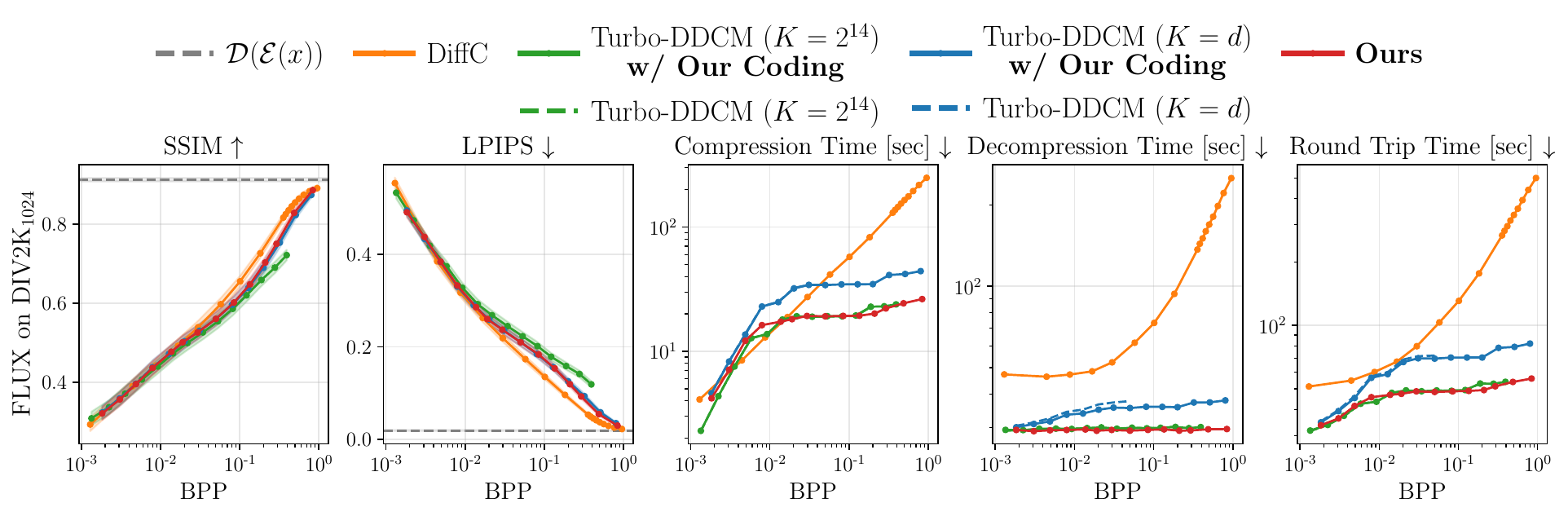}
    \end{subfigure}
    \vspace{-2em}
    \caption{\textbf{Complementary quantitative comparison of Figs.~\ref{fig:main_images_quantitative_comparisons} and~\ref{fig:supp_images_quantitative_comparisons} of image compression methods.} We compare \oursshort{} with training-free diffusion-based and other compression methods in terms of rate, distortion and perception. Distortion is measured by SSIM and LPIPS, runtime is measured by encoding (compression) time, decoding (decompression) time, and round-trip time (compression and decompression time). PSC is omitted due to its slow runtime. The first, second and third rows correspond to Kodak24, DIV2K and CLIC2020 datasets, respectively, evaluated on SD2.1 ($512 \times 512$, corresponding to $d = 2^{14}$). The fourth row corresponds to DIV2K dataset, evaluated on FLUX ($1024 \times 1024$, corresponding to $d = 2^{18}$). The dashed gray vertical line corresponds to the encoder-decoder bound. \oursshort{} achieves competitive quality, at the same time of being the fastest training-free method, without materializing the dictionary. Error bars represent the standard error of the mean (SEM).}
    \label{fig:supp_images_quantitative_comparisons_complementary}
\end{figure}

\FloatBarrier
\clearpage

\subsection{Image compression qualitative comparison}
Figures~\ref{fig:supp_images_qualitative_comparisons_kodak24},~\ref{fig:supp_images_qualitative_comparisons_div2k} and~\ref{fig:supp_images_qualitative_comparisons_clic2020} present additional SD~2.1 qualitative comparisons between diffusion-based training-free methods on Kodak24, DIV2K, and CLIC2020 datasets (at $512 \times 512$), respectively. Figure~\ref{fig:supp_images_qualitative_comparisons_div2k_flux} presents the same, but for FLUX on DIV2K dataset (at $1024 \times 1024$).

\begin{figure}[htbp]
    \centering
    \begin{subfigure}{\linewidth}
        \centering
        \includegraphics[width=\linewidth]{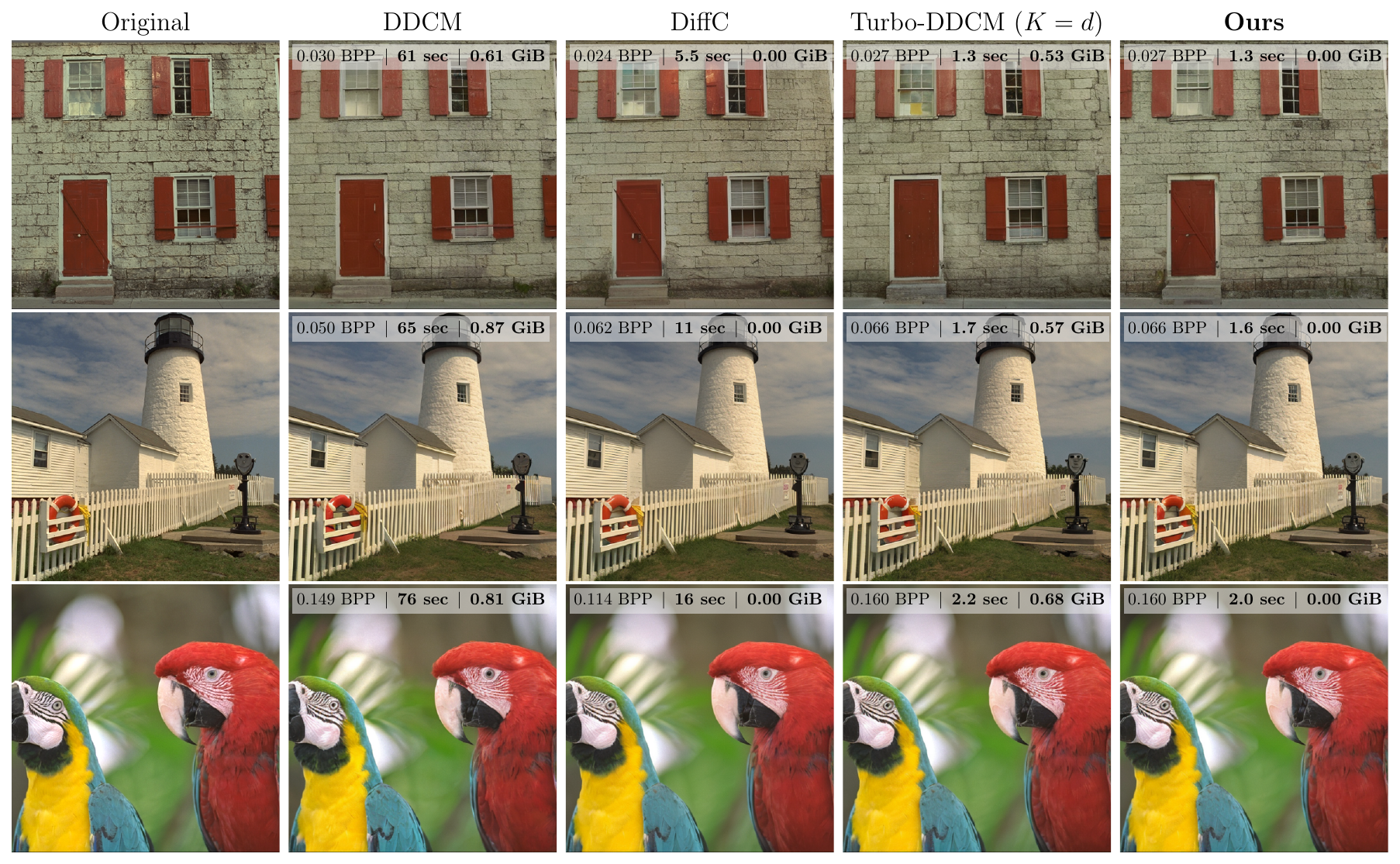}
    \end{subfigure}
    \vspace{-1.5em}
    \caption{\textbf{Qualitative comparison of training-free diffusion-based image compression methods (SD2.1).} The images are taken from the Kodak24 dataset (resized and center-cropped to resolution of $512 \times 512$).}
    \label{fig:supp_images_qualitative_comparisons_kodak24}
\end{figure}

\begin{figure}[htbp]
    \centering
    \begin{subfigure}{\linewidth}
        \centering
        \includegraphics[width=\linewidth]{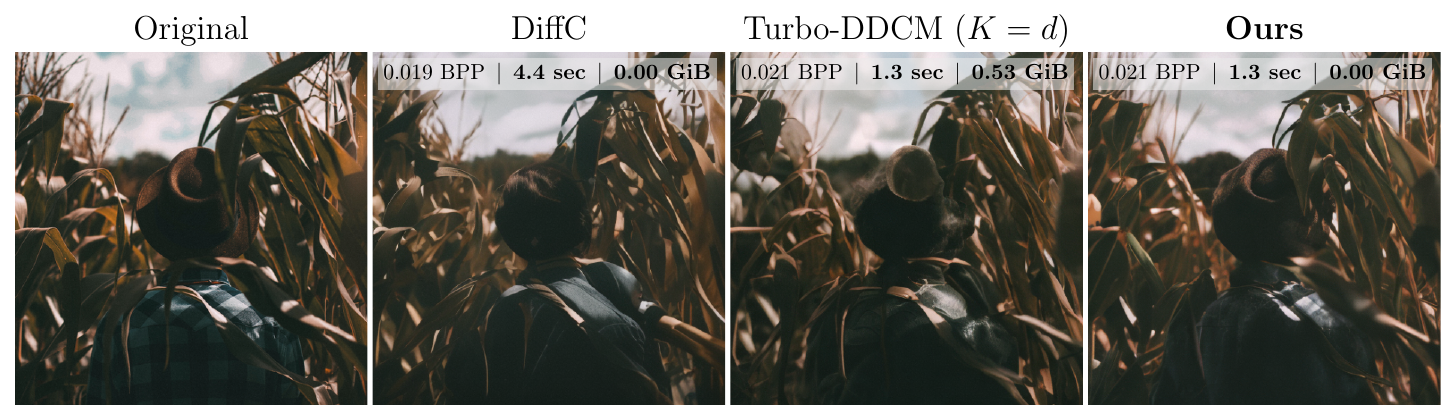}
    \end{subfigure}
    
    \vspace{-.5em}
    
    \begin{subfigure}{\linewidth}
        \centering
        \includegraphics[width=\linewidth]{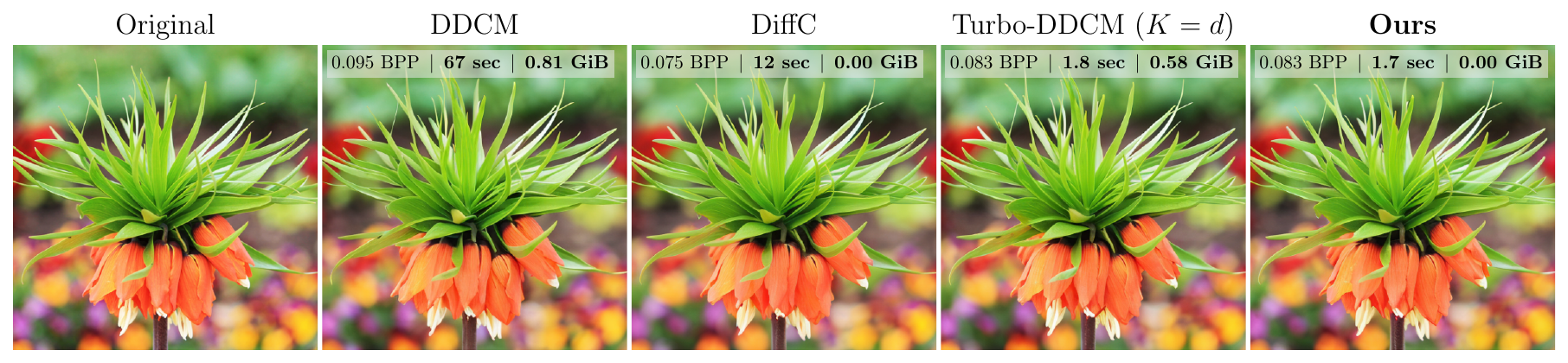}
    \end{subfigure}
    \vspace{-1.5em}
    \caption{\textbf{Qualitative comparison of training-free diffusion-based image compression methods (SD2.1).} The images are taken from the DIV2K dataset (resized and center-cropped to resolution of $512 \times 512$).}
    \label{fig:supp_images_qualitative_comparisons_div2k}
\end{figure}

\begin{figure}[htbp]
    \centering
    \begin{subfigure}{\linewidth}
        \centering
        \includegraphics[width=\linewidth]{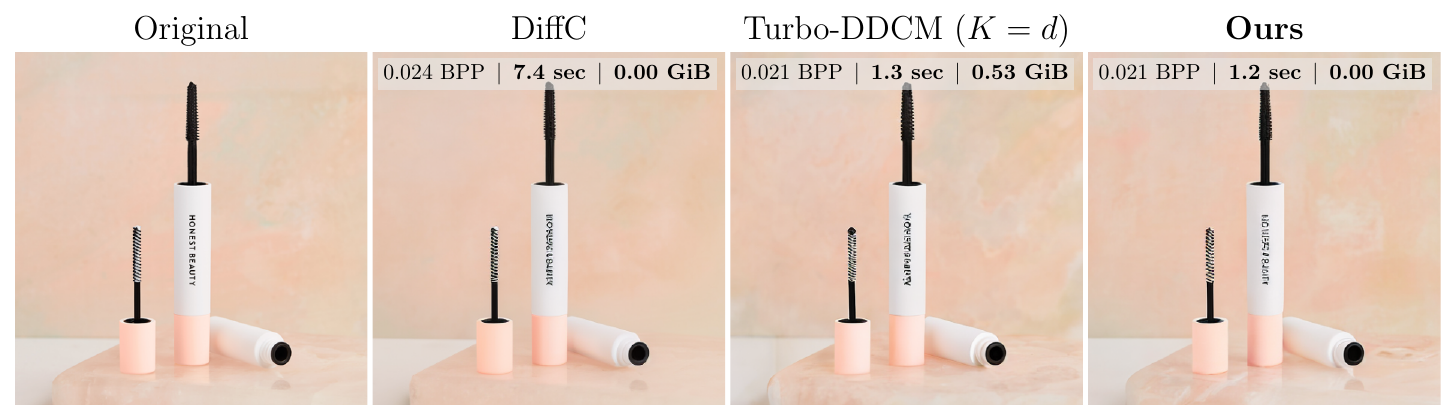}
    \end{subfigure}
    
    \vspace{-.5em}
    
    \begin{subfigure}{\linewidth}
        \centering
        \includegraphics[width=\linewidth]{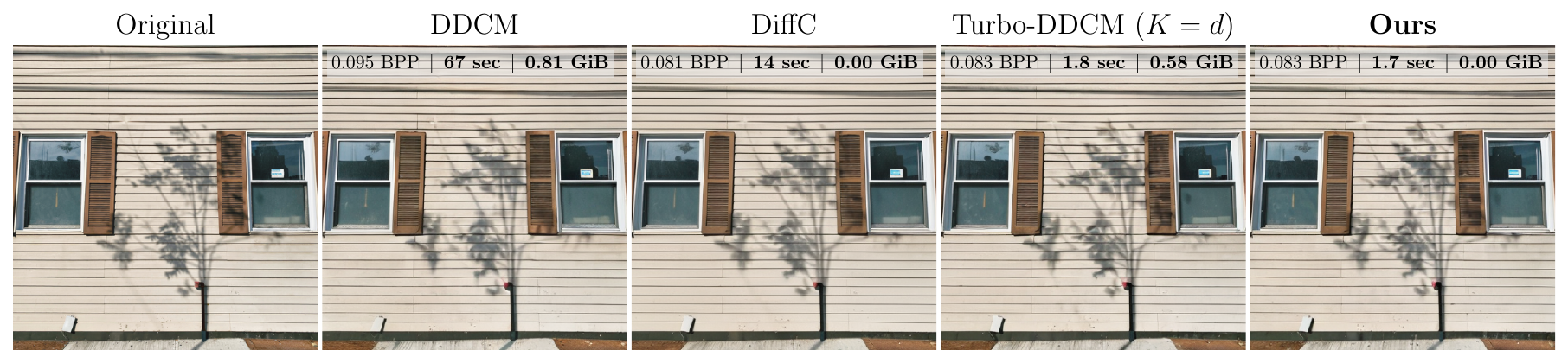}
    \end{subfigure}
    \vspace{-1.5em}
    \caption{\textbf{Qualitative comparison of training-free diffusion-based image compression methods (SD2.1).} The images are taken from the CLIC2020 dataset (resized and center-cropped to resolution of $512 \times 512$).}
    \label{fig:supp_images_qualitative_comparisons_clic2020}
\end{figure}

\begin{figure}[htbp]
    \centering
    \begin{subfigure}{\linewidth}
        \centering
        \includegraphics[width=\linewidth]{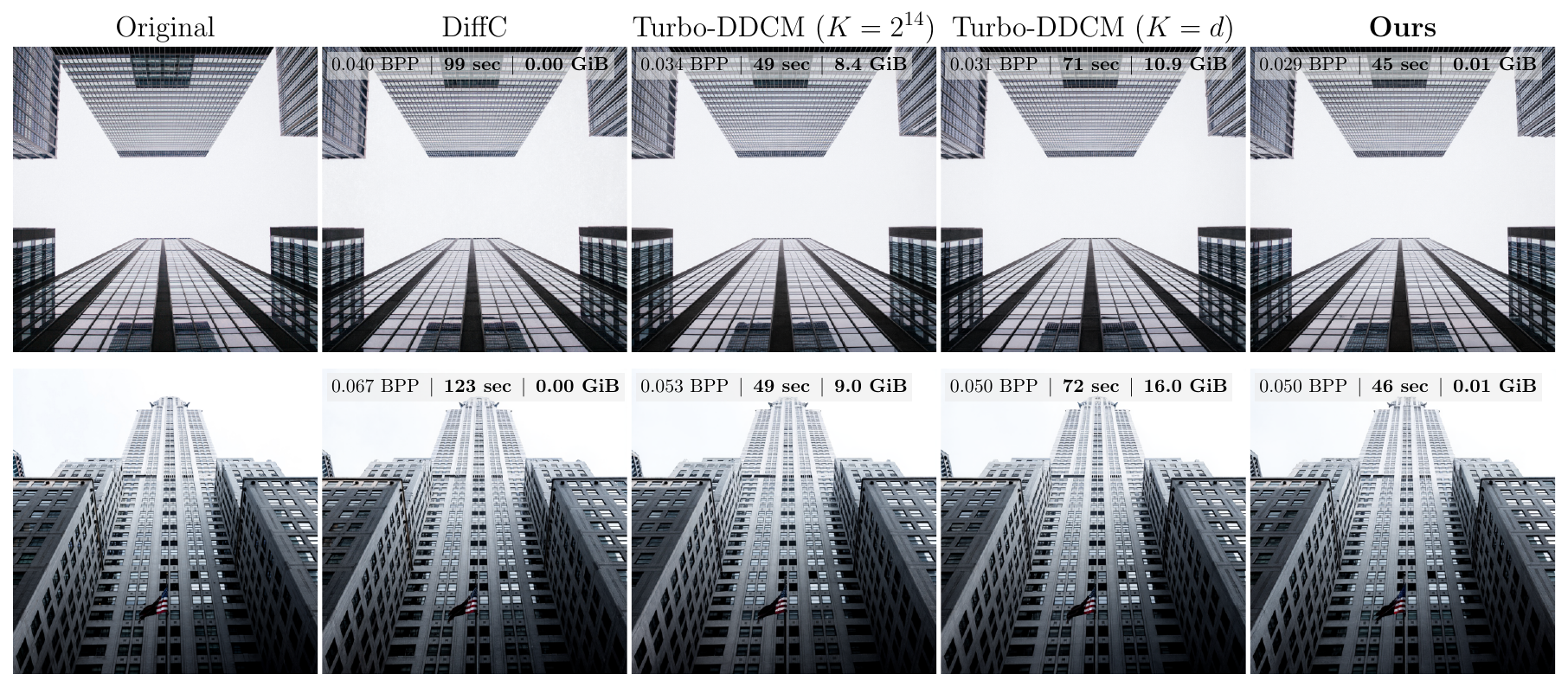}
    \end{subfigure}
    \vspace{-1.5em}
    \caption{\textbf{Qualitative comparison of training-free diffusion-based image compression methods (FLUX).} The images are taken from the DIV2K dataset (resized and center-cropped to resolution of $1024 \times 1024$).}
    \label{fig:supp_images_qualitative_comparisons_div2k_flux}
\end{figure}

\FloatBarrier
\clearpage

\subsection{Video compression quantitative comparison}
Figure~\ref{fig:supp_video_quantitative_comparisons} extends the video comparison of Fig.~\ref{fig:main_video_quantitative_comparisons}, to HEVC Classes B, C and D datasets, resized and center-cropped to resolution of $832\times480$ resolution. Because HEVC-D is natively $416\times240$ resolution, we upsample it to $832\times480$ to match the backbone's native operating resolution, but evaluate all quality metrics at the original $416\times240$ resolution rather than the upsampled size. The results follow the main paper trend: \oursshort{} OC achieves a competitive rate-distortion-perception performance to that of competing methods, while providing the highest measured throughput among the evaluated methods using the CF backbone, without materializing the dictionary.

\begin{figure}[htbp]
    \centering
    \begin{subfigure}{\linewidth}
        \centering
        \includegraphics[width=\linewidth]{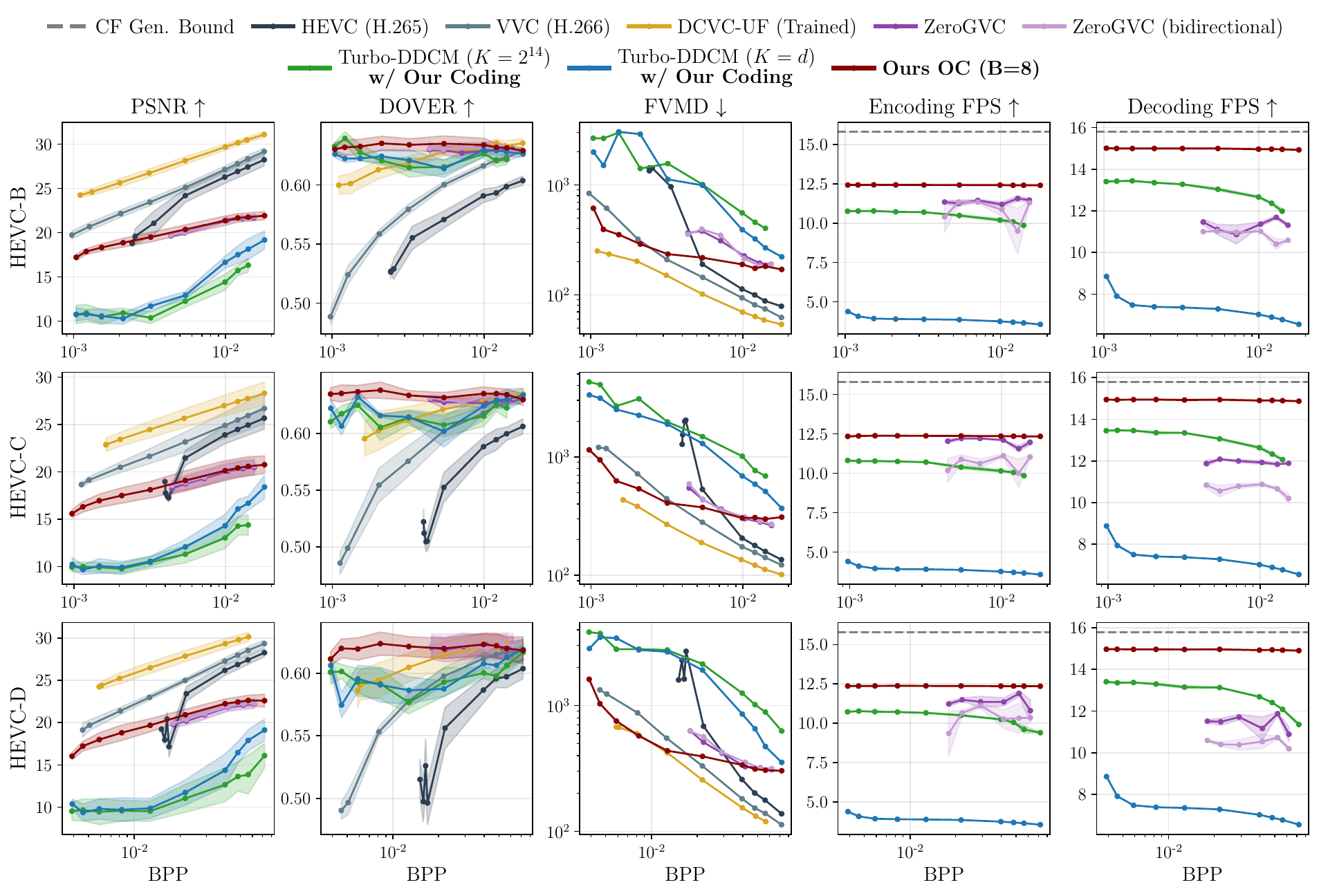}
    \end{subfigure}
    \vspace{-1.5em}
    \caption{\textbf{Quantitative comparison of video compression methods.} We compare \oursshort{} OC ($B=8$) with training-free methods using the same CF backbone, and other compression baselines in terms of rate, distortion and perception. Distortion is measured by PSNR, and perceptual quality is measured by DOVER and FVMD (presented in log-scale). Encoding and decoding FPS report compression and decompression throughput, respectively, and are compared among methods using the CF backbone for fair comparison. The first row corresponds to HEVC-B dataset, the second row corresponds to HEVC-C dataset, and the last row corresponds to HEVC-D dataset. \oursshort{} achieves competitive rate-distortion-perception performance among the CF-based training-free methods while providing the highest measured encoding and decoding throughput, without codebooks materialization. Error bars represent the standard error of the mean (SEM).}
    \label{fig:supp_video_quantitative_comparisons}
\end{figure}

Figure~\ref{fig:supp_video_quantitative_comparisons_complementary} presents complementary metrics of Figs.~\ref{fig:main_video_quantitative_comparisons} and~\ref{fig:supp_video_quantitative_comparisons}. Specifically, it presents the MS-SSIM~\citep{wang2003multiscale}, LPIPS~\citep{zhang2018unreasonable} and DISTS~\citep{ding2020image} distortion metrics, and perceptual quality metrics of FID~\citep{heusel2017gans} and FVD~\citep{unterthiner2018towards}.

\begin{figure}[htbp]
    \centering
    \begin{subfigure}{\linewidth}
        \centering
        \includegraphics[width=\linewidth]{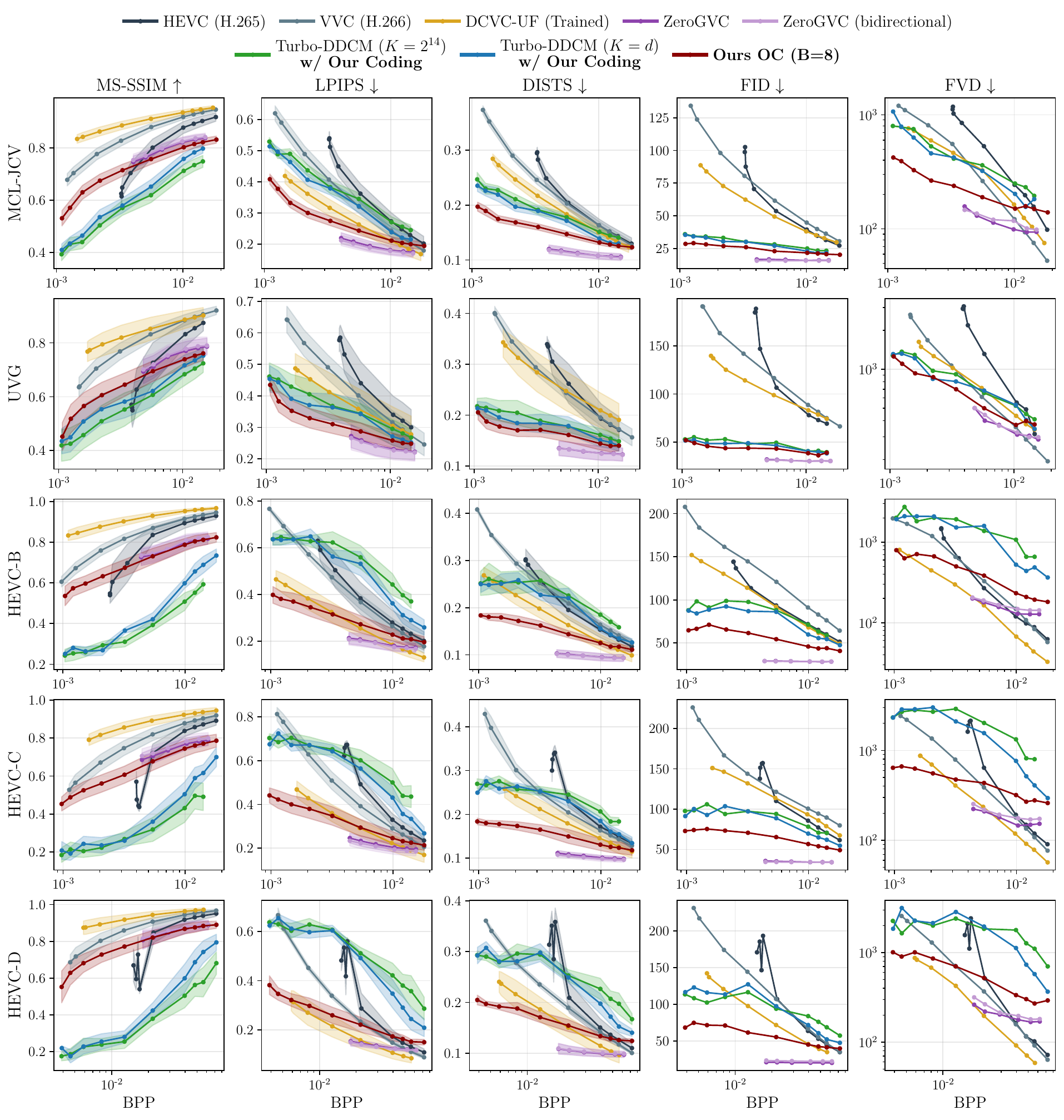}
    \end{subfigure}
    \vspace{-2em}
    \caption{\textbf{Complementary quantitative comparison of Figs.~\ref{fig:main_video_quantitative_comparisons} and~\ref{fig:supp_video_quantitative_comparisons} of video compression methods.} We compare \oursshort{} OC ($B=8$) with training-free methods using the same CF backbone and other compression baselines in terms of rate, distortion and perception. Distortion is measured by MS-SSIM, LPIPS and DISTS, perceptual quality is measured by FID and FVD (presented in log-scale). The first, second and third rows correspond to HEVC-B, HEVC-C and HEVC-D datasets, respectively. \oursshort{} achieves competitive rate-distortion-perception performance among the CF-based training-free methods. Error bars represent the standard error of the mean (SEM).}
    \label{fig:supp_video_quantitative_comparisons_complementary}
\end{figure}

\FloatBarrier
\clearpage

\subsection{Video compression qualitative comparison}
Figures~\ref{fig:supp_video_qualitative_comparisons_1},~\ref{fig:supp_video_qualitative_comparisons_2},~\ref{fig:supp_video_qualitative_comparisons_3},~\ref{fig:supp_video_qualitative_comparisons_4} and~\ref{fig:supp_video_qualitative_comparisons_5} present additional qualitative comparisons using the same setting and baselines as presented in Sec.~\ref{sec:vid_streaming_app}. We include spatial frames together with spatiotemporal slices to show both reconstruction fidelity and temporal continuity.

\begin{figure}[htbp]
    \centering
    \begin{subfigure}{\linewidth}
        \centering
        \includegraphics[width=\linewidth]{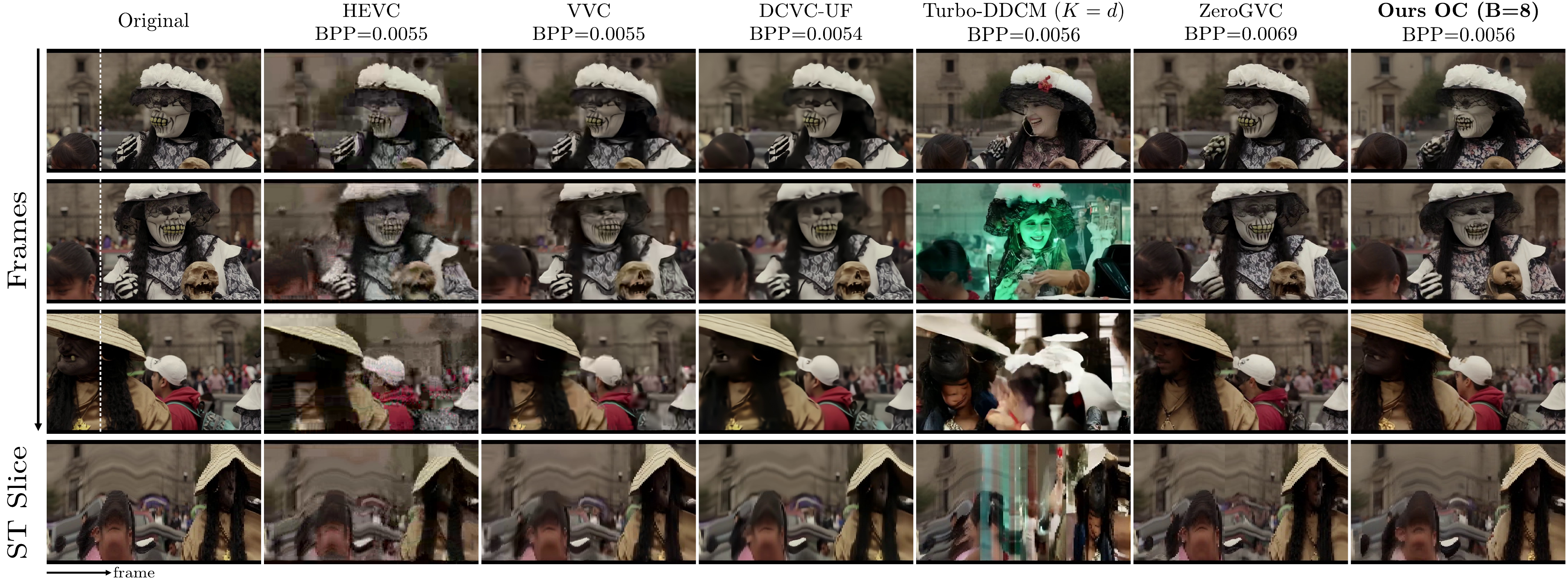}
    \end{subfigure}
    \vspace{-1.5em}
    \caption{\textbf{Qualitative comparison of video-compression methods.} The presented video is taken from MCL-JCV ($832\times480$) dataset. Our method produces comparable results on the rate-distortion-perception tradeoff, while being the fastest method utilizing the CF backbone. The space-time (ST) slice, which shows the temporal evolution of the pixels marked by the dashed line, reveals temporal discontinuities at ZeroGVC’s I-frame boundaries.}
    \label{fig:supp_video_qualitative_comparisons_1}
\end{figure}

\begin{figure}[htbp]
    \centering
    \begin{subfigure}{\linewidth}
        \centering
        \includegraphics[width=\linewidth]{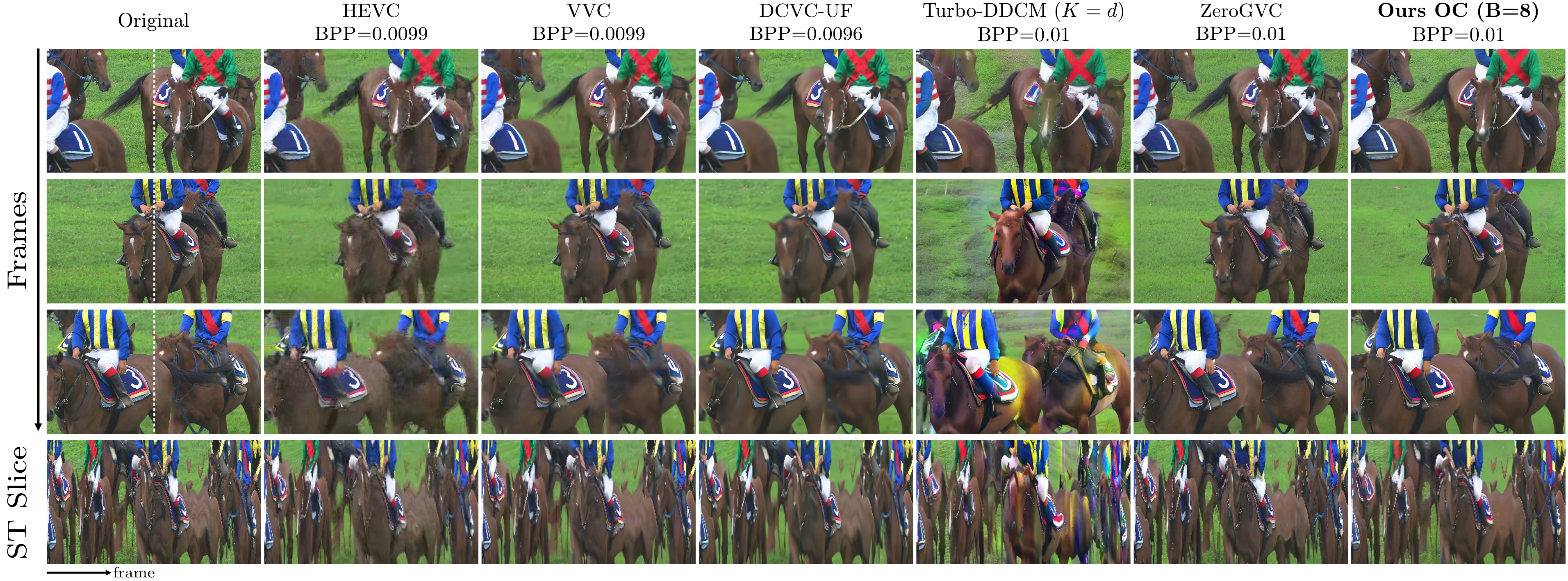}
    \end{subfigure}
    \vspace{-1.5em}
    \caption{\textbf{Qualitative comparison of video-compression methods.} The presented video is taken from HEVC-C ($832\times480$) dataset. Our method produces comparable results on the rate-distortion-perception tradeoff, while being the fastest method utilizing the CF backbone. The space-time (ST) slice, which shows the temporal evolution of the pixels marked by the dashed line, reveals temporal discontinuities at ZeroGVC’s I-frame boundaries.}
    \label{fig:supp_video_qualitative_comparisons_2}
\end{figure}

\begin{figure}[htbp]
    \centering
    \begin{subfigure}{\linewidth}
        \centering
        \includegraphics[width=\linewidth]{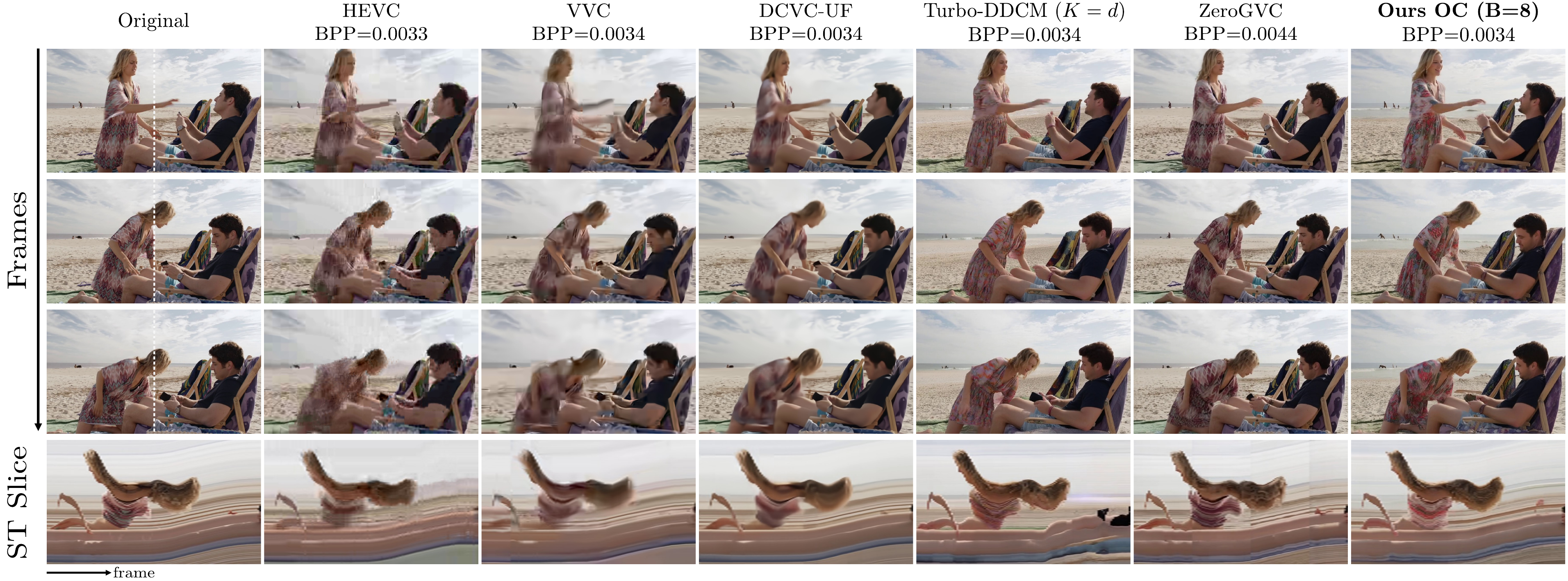}
    \end{subfigure}
    \vspace{-1.5em}
    \caption{\textbf{Qualitative comparison of video-compression methods.} The presented video is taken from MCL-JCV ($832\times480$) dataset. Our method produces comparable results on the rate-distortion-perception tradeoff, while being the fastest method utilizing the CF backbone. The space-time (ST) slice, which shows the temporal evolution of the pixels marked by the dashed line, reveals temporal discontinuities at ZeroGVC’s I-frame boundaries.}
    \label{fig:supp_video_qualitative_comparisons_3}
\end{figure}

\begin{figure}[htbp]
    \centering
    \begin{subfigure}{\linewidth}
        \centering
        \includegraphics[width=\linewidth]{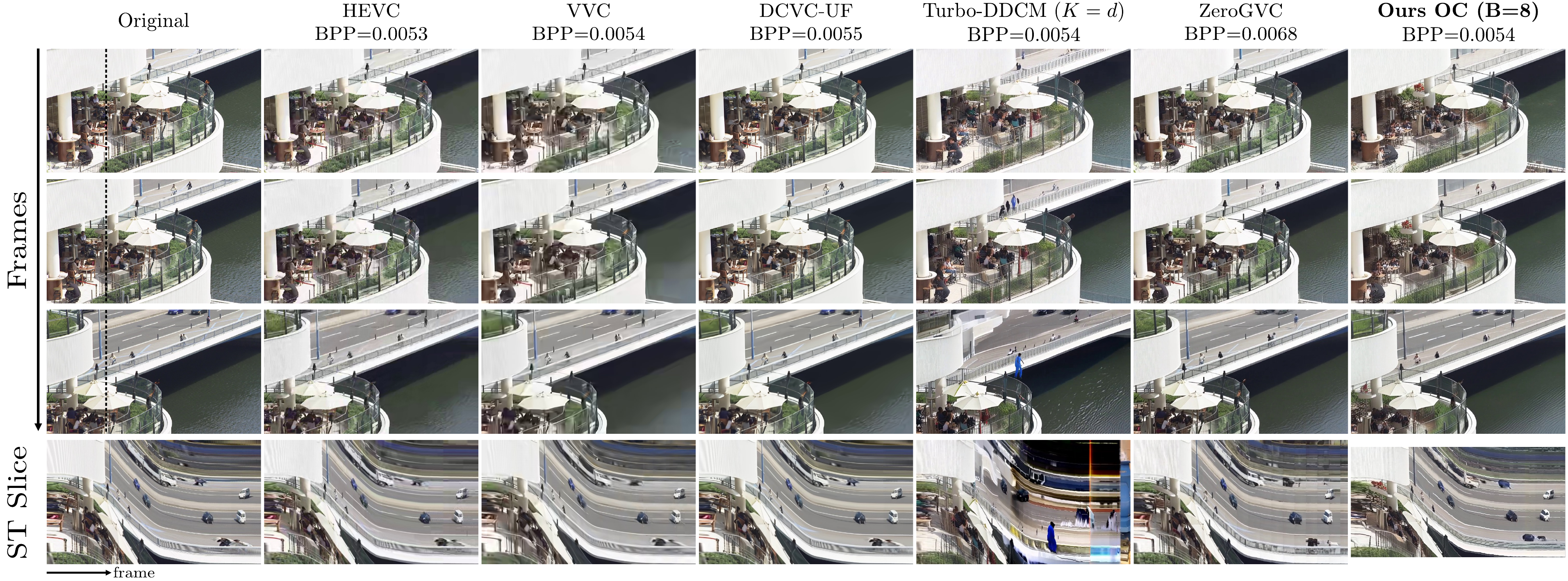}
    \end{subfigure}
    \vspace{-1.5em}
    \caption{\textbf{Qualitative comparison of video-compression methods.} The presented video is taken from HEVC-B ($832\times480$) dataset. Our method produces comparable results on the rate-distortion-perception tradeoff, while being the fastest method utilizing the CF backbone. The space-time (ST) slice, which shows the temporal evolution of the pixels marked by the dashed line, reveals temporal discontinuities at ZeroGVC’s I-frame boundaries.}
    \label{fig:supp_video_qualitative_comparisons_4}
\end{figure}

\begin{figure}[htbp]
    \centering
    \begin{subfigure}{\linewidth}
        \centering
        \includegraphics[width=\linewidth]{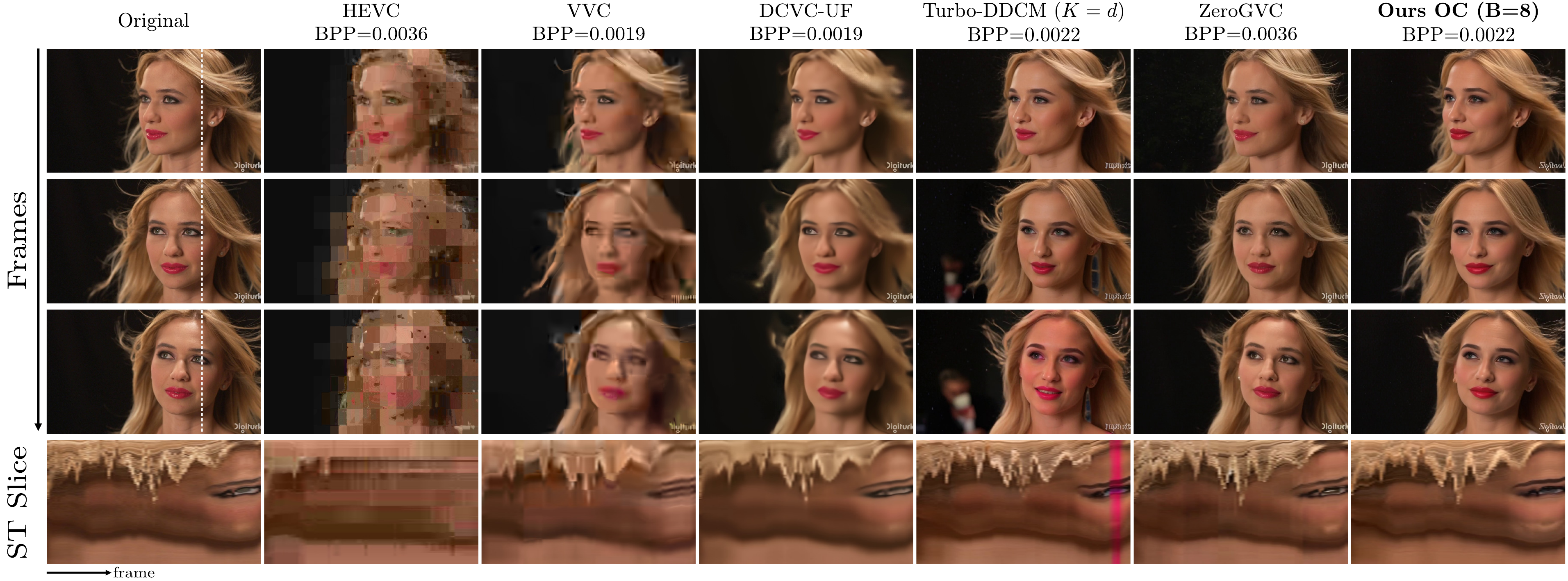}
    \end{subfigure}
    \vspace{-1.5em}
    \caption{\textbf{Qualitative comparison of video-compression methods.} The presented video is taken from UVG ($832\times480$) dataset. Our method produces comparable results on the rate-distortion-perception tradeoff, while being the fastest method utilizing the CF backbone. The space-time (ST) slice, which shows the temporal evolution of the pixels marked by the dashed line, reveals temporal discontinuities at ZeroGVC’s I-frame boundaries.}
    \label{fig:supp_video_qualitative_comparisons_5}
\end{figure}

\FloatBarrier
\clearpage

\subsection{Ablation of the overcomplete Hadamard dictionary}
We compare the single-basis \oursshort{} dictionary with the overcomplete (OC) variants using $B\in \{2,4,8\}$ randomized bases on the MCL-JCV dataset, following the same evaluation described in Sec.~\ref{sec:vid_streaming_app}. Figure~\ref{fig:supp_video_quantitative_ablation} presents the quantitative comparison, and shows that the OC variants improve both distortion and perceptual metrics across all evaluated bitrates, with $B=8$ providing the strongest results among the tested values.
Figure~\ref{fig:supp_video_qualitative_ablation} presents the qualitative comparison between the single-basis variant and $B=8$ variant at $\text{BPP}=0.01$. Both methods achieve compelling results for this low BPP value, however, the overcomplete dictionary variant is more consistent with the original video and exhibits higher visual realism

\begin{figure}[htbp]
    \centering
    \begin{subfigure}{\linewidth}
        \centering
        \includegraphics[width=\linewidth]{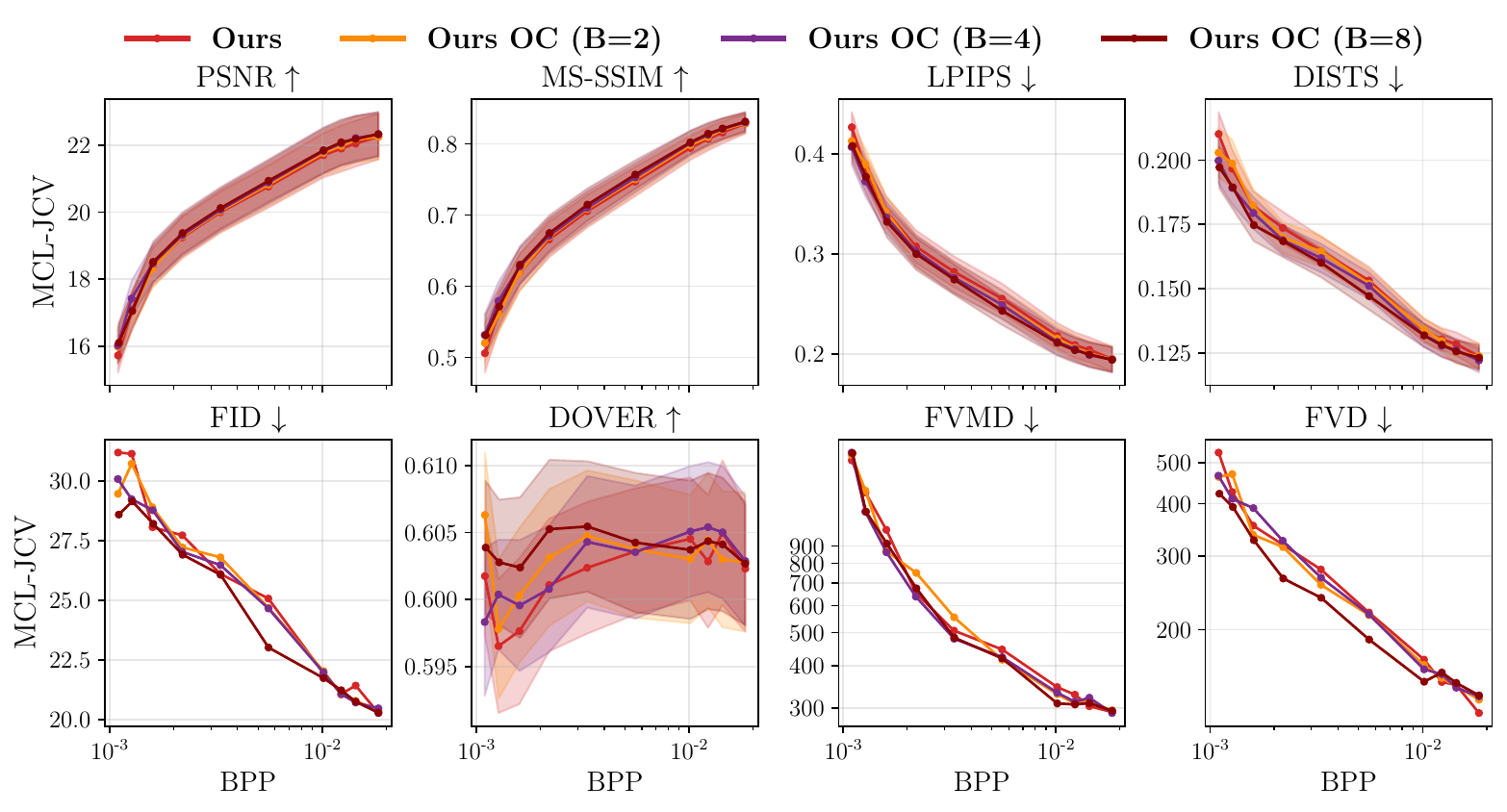}
    \end{subfigure}
    \vspace{-2em}
    \caption{\textbf{Quantitative comparison ablation between \oursshort{} and \oursshort{} OC ($B\in\{2,4,8\}$).} The overcomplete dictionary variant provides better performance across all distortion and perceptual quality metrics. Error bars represent standard error of the mean (SEM).}
    \label{fig:supp_video_quantitative_ablation}
\end{figure}

\begin{figure}[htbp]
    \centering
    \begin{subfigure}{\linewidth}
        \centering
        \includegraphics[width=\linewidth]{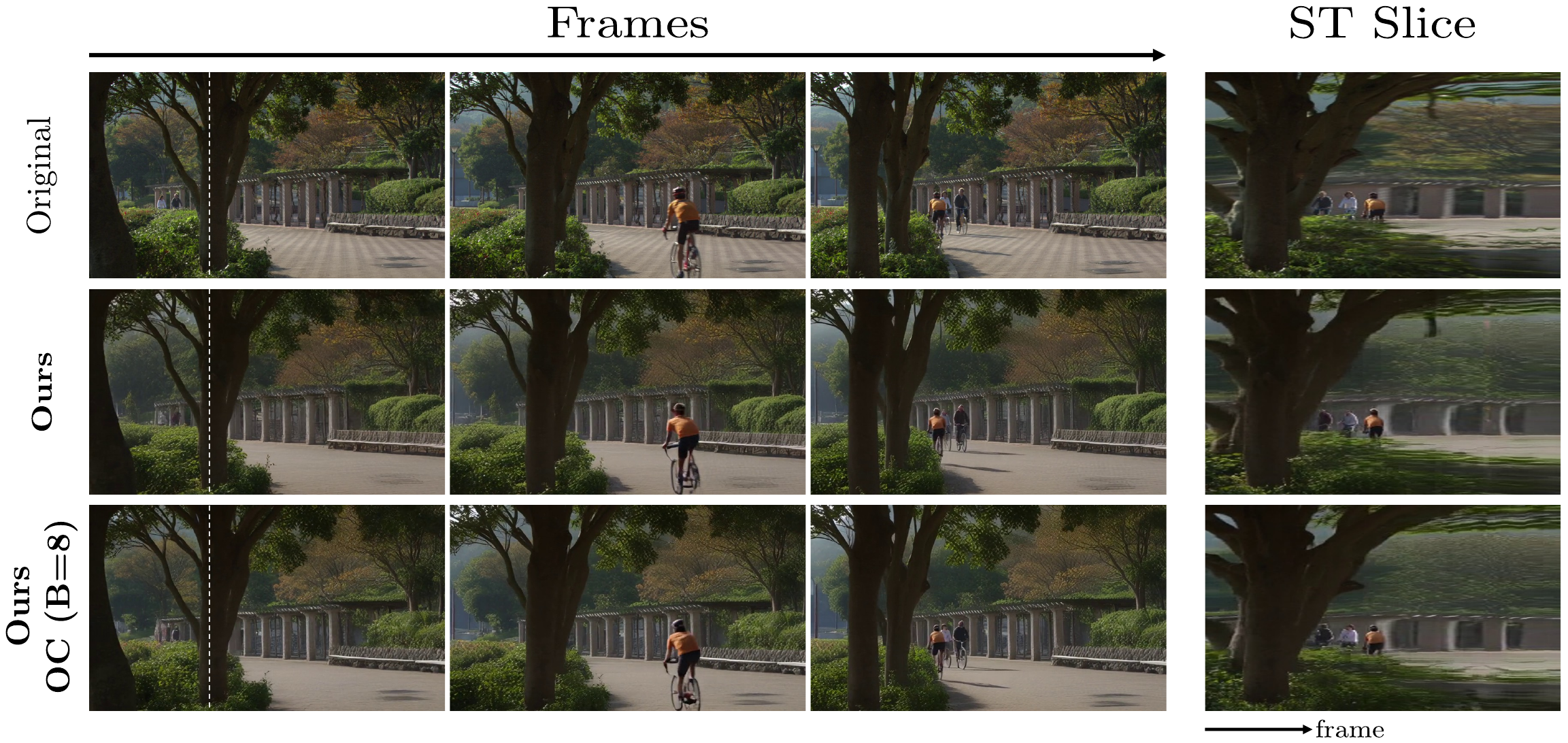}
    \end{subfigure}
    \vspace{-1.5em}
    \caption{\textbf{Qualitative comparison ablation between \oursshort{} and \oursshort{} OC ($B=8$).} The presented video is taken from HEVC-B ($832\times480$) dataset, for $\text{BPP}=0.01$. Both proposed methods achieve compelling results, however the perceptual quality and video consistency of the overcomplete dictionary is better.}
    \label{fig:supp_video_qualitative_ablation}
\end{figure}

\FloatBarrier
\clearpage

\subsection{Ablation of transpose and pseudo-inverse synthesis}
\label{app:trans_pinv}
When the dictionary is orthonormal, the columns of the selected sub-dictionary $\mathbf{C}_S \in \mathbb{R}^{d \times M}$ are mutually orthogonal ($\mathbf{C}_S^\top \mathbf{C}_S = \mathbf{I}_M$), so synthesizing via the pseudo-inverse $(\mathbf{C}_S^\dagger)^\top \boldsymbol{\alpha}_S = \mathbf{C}_S (\mathbf{C}_S^\top \mathbf{C}_S)^{-1} \boldsymbol{\alpha}_S$ collapses identically to the direct combination $\mathbf{C}_S \boldsymbol{\alpha}_S$ in~\eqref{eq:std-norm}. 
For the overcomplete dictionary, where selected atoms may originate from different bases are thus no longer mutually orthogonal. Therefore, we compare between synthesizing the noise direction via the pseudo-inverse $(\mathbf{C}_S^\dagger)^\top \boldsymbol{\alpha}_S$ and retaining the direct summation heuristic $\mathbf{C}_S \boldsymbol{\alpha}_S$ of~\citet{vaisman2026turbo}, on the UVG dataset. 
As shown in Fig.~\ref{fig:supp_video_quantitative_ablation_pinv}, both synthesis variants achieve nearly identical rate-distortion-perception performance across all evaluated BPP values, while the cosine similarity between their selected noise vectors remains highly aligned. We therefore use direct synthesis for \oursshort{} OC variant to avoid the additional solve.

\begin{figure}[htbp]
    \centering
    \begin{subfigure}{\linewidth}
        \centering
        \includegraphics[width=\linewidth]{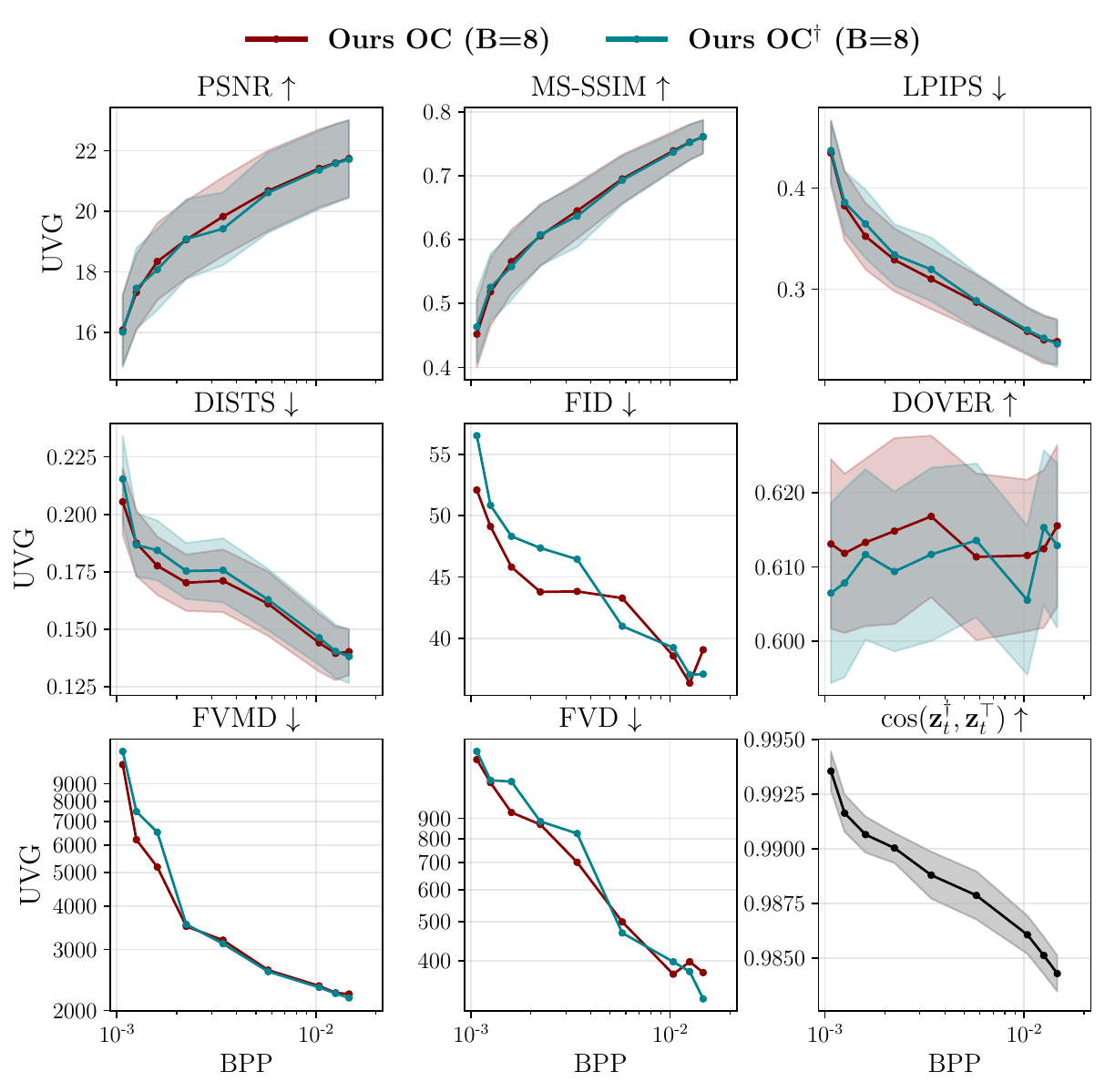}
    \end{subfigure}
    \vspace{-2em}
    \caption{\textbf{Quantitative comparison ablation between \oursshort{} OC utilizing the transpose variant, and \oursshort{} OC utilizing the pseudo-inverse, both with $B=8$.} Both variants achieve similar performance, and the cosine similarity between their selected noise vectors remains close to $1$, \ie the selected $\mathbf{z}_t^{\star}$ is almost aligned. Pseudo-inverse is denoted by $(\cdot)^{\dagger}$. Error bars represent standard error of the mean (SEM).}
    \label{fig:supp_video_quantitative_ablation_pinv}
\end{figure}

\FloatBarrier
\clearpage

\subsection{Ablation of cache refresh and first-step initialization}
We ablate two video-specific components of Sec.~\ref{sec:vid_streaming_app}, using the OC dictionary with $B=8$: the re-encoded-tail KV-cache refresh and the first-step initialization described in App.~\ref{app:sampling_formulations}. For cache handling, we compare our full method with ($i$) a hard reset at each continuation boundary (``w/o re-encoded-tail''), and ($ii$) a continuously growing cache without refresh (``w/ continuous cache'').

Figure~\ref{fig:supp_video_quantitative_ablation_mechanism} presents the quantitative comparison on the MCL-JCV dataset, following the same evaluation described in Sec.~\ref{sec:vid_streaming_app}.
The full refresh mechanism improves over both cache alternatives across the evaluated rates. The qualitative comparison in Fig.~\ref{fig:supp_video_qualitative_ablation_mechanism} shows the corresponding failure modes, for $\text{BPP}=0.01$: hard resets introduce discontinuities at segment boundaries, while the continuously growing cache produces severe artifacts once the context extends beyond the regime seen during training.

Additionally, we remove the first-step initialization mechanism (``w/o first step init.''). Although this variant can improve some aggregate quantitative metrics as seen in Fig.~\ref{fig:supp_video_quantitative_ablation_mechanism}, it struggles to consistently temporally align with the input video, as seen in Fig.~\ref{fig:supp_video_qualitative_ablation_mechanism}.

\begin{figure}[htbp]
    \centering
    \begin{subfigure}{\linewidth}
        \centering
        \includegraphics[width=\linewidth]{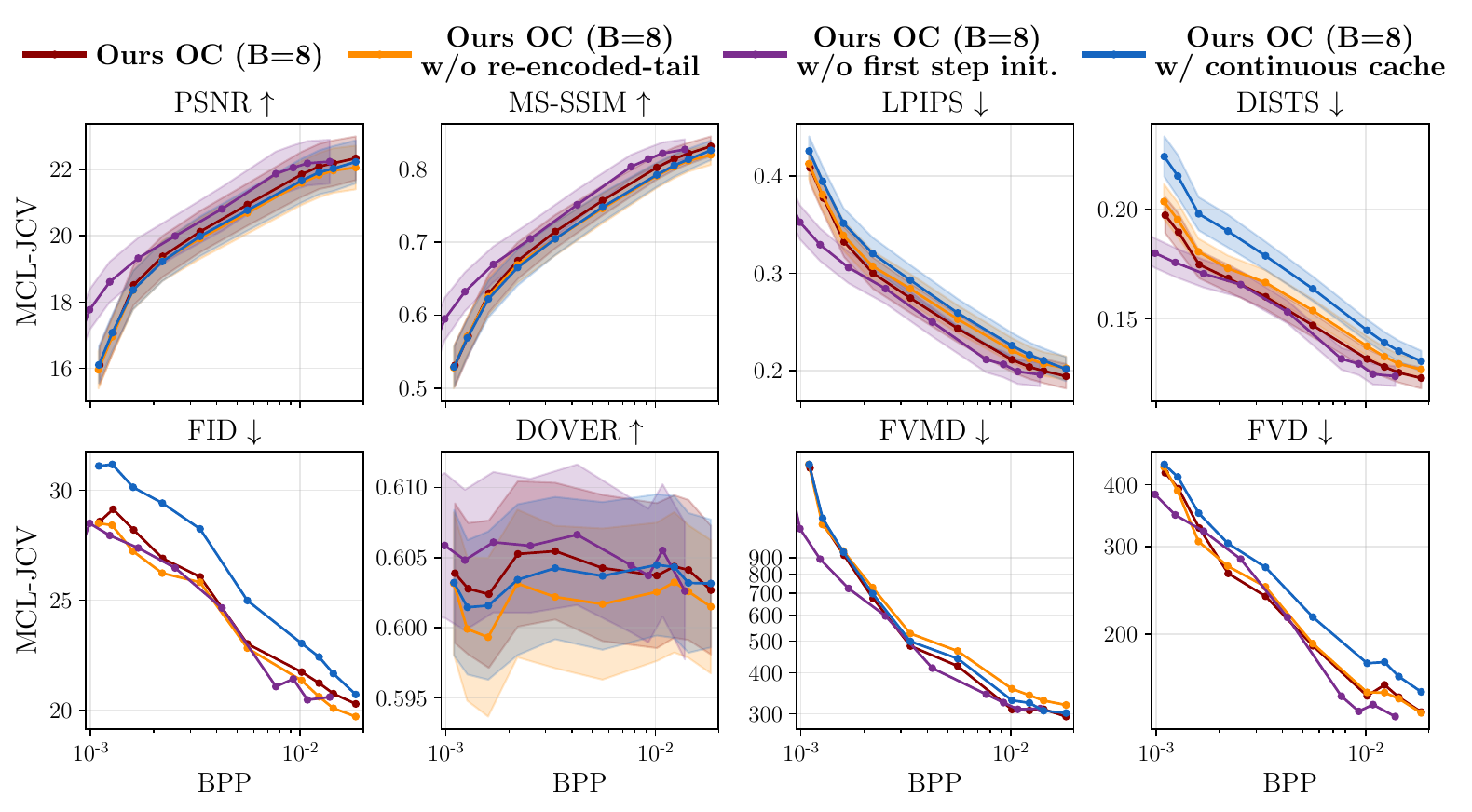}
    \end{subfigure}
    \vspace{-2em}
    \caption{\textbf{Quantitative comparison ablation between \oursshort{} OC ($B=8$), \oursshort{} OC ($B=8$) without the re-encoded-tail mechanism for both hard cache reset -- ``w/o re-encoded-tail'', and without cache reset -- ``w/ continuous cache'', and \oursshort{} OC ($B=8$) without the first timestep initialization mechanism.} \oursshort{} OC ($B=8$) achieves better with the re-encoded tail mechanism. Without the first timestep initialization mechanism provides better metrics, however as can be seen in Fig.~\ref{fig:supp_video_qualitative_ablation_mechanism}, it struggles to consistently align with the input video. Error bars represent standard error of the mean (SEM).}
    \label{fig:supp_video_quantitative_ablation_mechanism}
\end{figure}

\begin{figure}[htbp]
    \centering
    \begin{subfigure}{\linewidth}
        \centering
        \includegraphics[width=\linewidth]{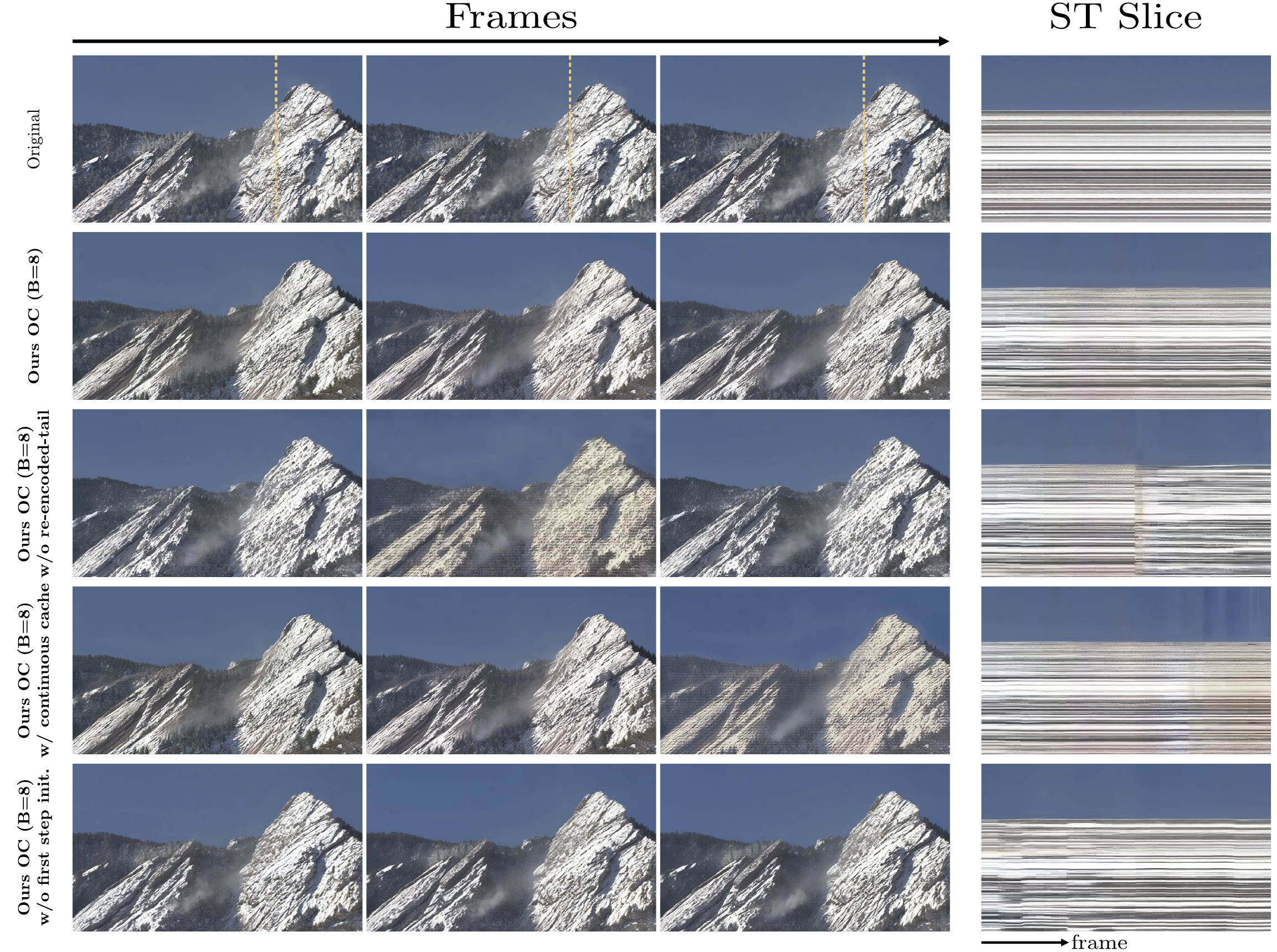}
    \end{subfigure}
    \vspace{-1.5em}
    \caption{\textbf{Qualitative comparison ablation between \oursshort{} OC ($B=8$), \oursshort{} OC ($B=8$) without the re-encoded-tail mechanism for both hard cache reset -- ``w/o re-encoded-tail'', and without cache reset -- ``w/ continuous cache'', and \oursshort{} OC ($B=8$) without the first timestep initialization mechanism.} The presented video is taken from MCL-JCV ($832\times480$) dataset, for $\text{BPP}=0.01$. While \oursshort{} OC ($B=8$) is aligned with the input video, without the re-encoded-tail mechanism when a new segment begins and a KV cache refresh occurs -- there is a temporal discontinuity which can be observed in the spatiotemporal slice, and artifacts which can be observed in the second presented frame. For the continuous cache variant, \ie without cache reset, there are severe artifacts caused as this size of a cache/context is out-of-distribution for the model, which could be observed in the third presented frame. Additionally, without the first timestep initialization the method struggles to consistently align with the input video.}
    \label{fig:supp_video_qualitative_ablation_mechanism}
\end{figure}

\FloatBarrier
\clearpage

%% file: supplementary/sampling_formulations.tex
\section{Sampling formulations}
\label{app:sampling_formulations}

Section~\ref{sec:bg:ddcm} writes the stochastic reverse transition in a sampler-agnostic form. Here we give the exact image and video updates used in our experiments.

\myparagraph{Image formulation: DDPM ancestral sampling.}
Following~\citet{vaisman2026turbo}, our image backbones (SD~2.1, FLUX) share the standard DDPM posterior mean~\citep{ho2020denoising}. Writing $\gamma_t \in (0,1)$ for the per-step signal-retention coefficient and $\bar\alpha_t = \prod_{i\le t}\gamma_i$ for its cumulative product; we use $\gamma_t$ rather than the more common $\alpha_t$, to distinguish the schedule from the sparse coefficient vector $\boldsymbol{\alpha}_t$. The mean and variance of the DDPM schedule are
\begin{equation}
\label{eq:mu_ddpm}
\boldsymbol{\mu}_t(\mathbf{x}_t) = \frac{\sqrt{\bar\alpha_{t-1}}\,(1-\gamma_t)}{1-\bar\alpha_t}\,\hat{\mathbf{x}}_{0|t} + \frac{\sqrt{\gamma_t}\,(1-\bar\alpha_{t-1})}{1-\bar\alpha_t}\,\mathbf{x}_t, \qquad \sigma_t = \sqrt{1-\gamma_t}.
\end{equation}
This is exactly the affine map of $\mathbf{x}_t$ referenced in Sec.~\ref{sec:bg:ddcm}: linear in $\mathbf{x}_t$ once the network's estimate $\hat{\mathbf{x}}_{0|t}$ is fixed by the forward pass at step $t$.

The last $N+1$ steps of the trajectory (Sec.~\ref{sec:method:rate}) are instead run deterministically, since they carry no bits and the decoder reproduces them unassisted. This replaces~\eqref{eq:mu_ddpm} with the $\sigma_t=0$ DDIM update~\citep{song2021denoising},
\begin{equation}
\label{eq:mu_ddim}
\boldsymbol{\mu}_t^{\mathrm{DDIM}}(\mathbf{x}_t) = \sqrt{\bar\alpha_{t-1}}\,\hat{\mathbf{x}}_{0|t} + \sqrt{1-\bar\alpha_{t-1}}\cdot\frac{\mathbf{x}_t - \sqrt{\bar\alpha_t}\,\hat{\mathbf{x}}_{0|t}}{\sqrt{1-\bar\alpha_t}},
\end{equation}
which still reads off the same network estimate $\hat{\mathbf{x}}_{0|t}$, only recombined with $\mathbf{x}_t$ along the noiseless direction consistent with it.

\myparagraph{Video formulation: consistency sampling.}
The causal video backbone distills a bidirectional teacher into a few-step AR student~\citep{zhu2026causal,huang2026self} parameterized as a consistency model~\citep{song2023consistency}. Specifically, a single forward evaluation $f_\theta(\mathbf{x}_t, t)$ maps a noisy chunk at any noise level $t$ directly to the origin of its probability flow trajectory, satisfying $f_\theta(\mathbf{x}_t, t) \approx f_\theta(\mathbf{x}_{t'}, t') \approx \hat{\mathbf{x}}_0$ for all $t \ne t'$, subject to the boundary condition $f_\theta(\mathbf{x}_\epsilon, \epsilon) = \mathbf{x}_\epsilon$ at the terminal noise scale $\epsilon$. In practice, the distilled student yields an empirical approximation of this mapping, which we denote as $\hat{\mathbf{x}}_{0|t} := f_\theta(\mathbf{x}_t, t)$ to maintain consistent notation. Unlike the DDPM transition in~\eqref{eq:mu_ddpm}, which interpolates between the current state $\mathbf{x}_t$ and the estimate $\hat{\mathbf{x}}_{0|t}$, consistency sampling constructs the transition drift toward the next discretized step $\tau_{n+1}$ by directly rescaling the predicted clean chunk:
\begin{equation}
\label{eq:mu_consistency}
\boldsymbol{\mu}_{\tau_n}(\mathbf{x}_{\tau_{n}}) := (1 - \tau_{n+1})\, f_\theta(\mathbf{x}_{\tau_{n}}, \tau_{n}) = (1 - \tau_{n+1})\, \hat{\mathbf{x}}_{0|\tau_{n}}.
\end{equation}
The residual noise vector $\mathbf{z}_{\tau_n}^\star$ is selected against $\hat{\mathbf{x}}_{0|\tau_{n}}$ and injected at the target schedule level $\tau_{n+1}$ scaled by $\sigma_{\tau_n} := \tau_{n+1}$:
\begin{equation}
\label{eq:mu_consistency_reinject}
\mathbf{x}_{\tau_{n+1}} = \boldsymbol{\mu}_{\tau_n}(\mathbf{x}_{\tau_{n}}) + \sigma_{\tau_n}\,\mathbf{z}_{\tau_n}^\star.
\end{equation}

\myparagraph{Initializing the first step.}
Both formulations select noise against the residual $\mathbf{r}_t := \mathbf{x}_{\mathrm{src}} - \hat{\mathbf{x}}_{0|t}$ (Sec.~\ref{sec:bg:turbo}), which needs $\hat{\mathbf{x}}_{0|t}$ from the network. Our video backbone's consistency schedule is just $n \in \{1,\ldots,4\}$ steps end to end, and $\hat{\mathbf{x}}_{0|\tau_{n}}$ in~\eqref{eq:mu_consistency_reinject} is only defined once a previous step has run; at $n=0$, \ie before the model evaluation, there is no previous step, so no network prediction exists yet to form $\mathbf{r}_t$ against. \citet{ohayon2025compressed} and~\citet{vaisman2026turbo} suggest using a shared seed to generate random Gaussian noise for the initial noise $\mathbf{x}_T$. We instead use an approximation of the posterior mean: the source chunk's own spatial mean, $\boldsymbol{\mu}_{\mathrm{src}} := \mathrm{mean}(\mathbf{x}_{\mathrm{src}})\mathbf{1}$, giving a proxy residual of $\mathbf{x}_{\mathrm{src}} - \boldsymbol{\mu}_{\mathrm{src}}$ for that one step only; every subsequent step reverts to the network's own $\hat{\mathbf{x}}_{0|t}$. We ablate this initialization in App.~\ref{app:additional_results}.

\myparagraph{Bit-carrying steps.}
Section~\ref{sec:method:rate} writes the number of bit-carrying steps as $S = T-N-1$ for images and $S = T$ for video (Eq.~\eqref{eq:bpp}). The $N+1$ accounts for the tail described above: the image formulation spends its last $N+1$ steps on the deterministic, bit-free $\boldsymbol{\mu}_t^{\mathrm{DDIM}}$ update~\eqref{eq:mu_ddim} rather than the bit-carrying $\boldsymbol{\mu}_t$ of~\eqref{eq:mu_ddpm} -- including the final transition $t=1\to0$, where $\bar\alpha_0=1$ collapses~\eqref{eq:mu_ddim} to the network's bare estimate $\hat{\mathbf{x}}_{0|1}$, carrying no bits regardless of $N$. The video formulation has no deterministic tail at all: with only four consistency steps to begin with, every one of them -- including the first, bootstrapped from $\boldsymbol{\mu}_{\mathrm{src}}$ above -- selects and transmits noise.

\clearpage

%% file: supplementary/bitstream_protocol.tex
\section{Bitstream protocol}
\label{app:bitstream_protocol}

At each bit-carrying step, the encoder transmits an unordered support of $M$ selected atoms together with their quantized coefficients. Listing the indices independently would require $M\lceil\log_2K\rceil$ bits. This approach is wasteful as it redundantly allocates different codewords all $M!$ permutations of the same support. We instead encode the sorted support by its rank among the $\binom{K}{M}$ possible $M$-subsets, requiring $\lceil\log_2\binom{K}{M}\rceil$ support bits plus $MC$ bits for the corresponding quantized coefficients in $\mathcal{V}$.

\myparagraph{Lexicographical combination encoding and decoding.}
Let $c = [c_1, \dots, c_M]$, $c_1 < \dots < c_M$, be the sorted combination of selected indices drawn from $\{0, \dots, K-1\}$. Its lexicographic rank is
\begin{equation}
\mathrm{rank}(c) \;=\; \sum_{i=1}^{M} \sum_{x = c_{i-1}+1}^{c_i - 1} \binom{K - x - 1}{M - i}, \qquad c_0 := -1,
\label{eq:lex_rank}
\end{equation}
\ie for every index $x$ skipped before reaching $c_i$, we add the number of combinations that would have been reachable had $x$ been chosen instead. This gives a bijection between combinations and integers in $[0, \binom{K}{M})$, so $\mathrm{rank}(c)$ alone determines $c$ uniquely.

Decoding inverts this by the same accounting in reverse: starting from $x=0$, we repeatedly test whether the next binomial term still fits inside the remaining rank,
\begin{equation}
\text{while } \binom{K-x-1}{M-i} \le \mathrm{rank}:\quad
\mathrm{rank} \leftarrow \mathrm{rank} - \binom{K-x-1}{M-i},\;\; x
\leftarrow x+1,
\label{eq:lex_decode}
\end{equation}
and take the first $x$ for which the term exceeds the remaining rank as $c_i$, before moving on to $c_{i+1}$. This recovers every element of $c$ in increasing order and is the exact inverse of~\eqref{eq:lex_rank}.

\myparagraph{Group-factorized encoding.} Rather than rank the full support of size $K$ in one piece, we partition it into $G$ equal groups of $K_G := K/G$ elements each and factorize the bitstream by group: writing $m_g$ for the number of atoms selected from group $g$, so $\sum_{g=1}^{G}m_g=M$, we first encode the count vector $(m_1,\ldots,m_G)$ itself, then rank each group's own local support of size $K_G$ separately with the machinery below.

A stars-and-bars map turns the count vector into $G-1$ strictly increasing separators,
\begin{equation}
s_g = \sum_{q=1}^{g}m_q + (g-1), \qquad g=1,\ldots,G-1,
\label{eq:oc-separators}
\end{equation}
which is exactly a $(G-1)$-subset of an $(M+G-1)$-element universe and is ranked by the same lexicographic construction as below, at a cost of
\begin{equation}
L_{\mathrm{counts}} = \left\lceil\log_2\binom{M+G-1}{G-1}\right\rceil
\label{eq:oc-counts-length}
\end{equation}
bits. Every nonempty group then contributes its own rank of an $m_g$-subset of $\{0,\ldots,K_G-1\}$, plus $m_g C$ coefficient bits, for a total slot length of
\begin{equation}
L_{\mathrm{grouped}} = L_{\mathrm{counts}} + \sum_{g=1}^{G}\left(\left\lceil\log_2\binom{K_G}{m_g}\right\rceil + m_g C\right);
\label{eq:oc-factorized-length}
\end{equation}
empty groups contribute neither term. Decoding reverses the order: unrank the separators to recover every $m_g$, which fixes the exact length of every local payload that follows, then unrank each nonempty group's own support and coefficients and re-form the flat index $\ell=gK_G+j$. No group identity is transmitted on its own -- it is implicit in which local rank a given index came from.

Setting $G=1$ recovers a single flat rank over all $K$ coordinates ($K_G=K$, $L_{\mathrm{counts}}=0$). The overcomplete dictionary of Sec.~\ref{sec:method:overcomplete} sits at the other end, $G=B$, with each of the $B$ randomized bases forming one group of $K_G=d$ atoms (total support $K=Bd$); a flat index over the full $Bd$-atom union would push the arithmetic below to a combinadic space $B$ times larger than any other method here ever ranks (App.~\ref{app:overcomplete_hadamard} gives the geometry in full).

\myparagraph{Implementation details.}
A direct evaluation of~\eqref{eq:lex_rank} and~\eqref{eq:lex_decode} requires up to $K$ binomial-coefficient evaluations per encoded/decoded combination. A natural way to avoid recomputing each one from scratch is to precompute and cache the full table of values $\binom{K-x-1}{M-i}$ for every $(x,i)$ once, ahead of time, and simply look up the relevant entries at runtime~\citep{vaisman2026turbo}. This is effective as long as $K$ and $M$ remain modest, but the table itself has size $O(K \cdot M)$: at the scale required for high-resolution latent spaces -- \eg $K$ on the order of $2^{18}$, with $M$ occasionally exceeding $10^3$ -- such a table would need on the order of $10^{8}$ entries, making it impractical to build, store, or keep resident in memory.

We instead avoid materializing any table with an incremental recurrence: consecutive binomial terms in~\eqref{eq:lex_rank} and~\eqref{eq:lex_decode}, writing $n,k$ for $(K-x-1, M-i)$, are related by two exact, $O(1)$ identities (Pascal's rule, applied in each of its two directions):
\begin{align}
\binom{n-1}{k-1} &= \binom{n}{k}\cdot \frac{k}{n}, &
\binom{n-1}{k} &= \binom{n}{k}\cdot \frac{n-k}{n}.
\label{eq:pascal_recurrence}
\end{align}
Starting from a single initial value, we advance a running binomial coefficient through the same $(n,k)$ lattice that~\eqref{eq:lex_rank} and~\eqref{eq:lex_decode} would otherwise re-derive from scratch at every term, using whichever of the two recurrences in~\eqref{eq:pascal_recurrence} matches the step being taken. This yields exactly the same rank (and its exact inverse) using $O(K)$ total $O(1)$ updates and $O(1)$ auxiliary memory, with no dependence on a precomputed table at all. Because the intermediate binomial coefficients can still grow to several hundred bits at the $K$, $M$ values we consider, the running value is maintained with an arbitrary-precision integer library rather than fixed-width integers. We verified this procedure produces bit-identical encodings and decodings to the direct formula at every $(K,M)$ pair we tested, while remaining tractable at scales where a precomputed table is
not.

A second table-free option, the hockey-stick identity (see below for additional details), rewrites~\eqref{eq:lex_rank}'s inner sum in closed form: for each selected index $c_i$, the count skipped since $c_{i-1}$ becomes a single difference of two fresh binomial evaluations,
\begin{equation}
\binom{K-c_{i-1}-1}{M-i+1} - \binom{K-c_i}{M-i+1},
\label{eq:hockey_stick}
\end{equation}
summed over $i=1,\ldots,M$ to encode; decoding inverts~\eqref{eq:hockey_stick} by a binary search over $c_i$ at each of the $M$ steps. This costs $O(M)$ fresh binomial evaluations to encode and $O(M\log K)$ to decode, against the recurrence's flat $O(K)$ -- and since evaluating $\binom{n}{k}$ from scratch is not itself $O(1)$ in $k$, which mechanism wins depends on how large $M$ is next to $K$. Both run on the CPU in parallel with the diffusion model's own GPU forward pass, so whichever one is cheap for the $M$ regime it runs at stays negligible next to the model's own step cost. We use both in this paper: for images, whose $M$ sweeps well into the regime where hockey-stick is no longer competitive, we use the recurrence for both ranking and unranking. For video, \oursshort{} and both Turbo-DDCM variants alike, both ranking and unranking use the hockey-stick identity throughout at every $M$ we sweep.

For completeness, in the following, we summarize the hockey-stick identity and the derivation of~\eqref{eq:hockey_stick}. The hockey-stick identity is defined as
\begin{equation}
    \sum_{y=k}^n \binom{y}{k} = \binom{n+1}{k+1}.
\label{hs_identity}
\end{equation}
In~\eqref{eq:lex_rank} we have $\sum_{x=c_{i-1}+1}^{c_i-1} \binom{K-x-1}{M-i}$, where the top index varies, not the bottom. Writing $k = M-i$ and substituting $y = K-x-1$ we obtain
\begin{equation*}
    \sum_{y=K-c_i}^{K-c_{i-1} -2} \binom{y}{k} = \sum_{y=k}^{K-c_{i-1} - 2} \binom{y}{k} - \sum_{y=k}^{K-c_i - 1} \binom{y}{k} = \binom{K - c_{i-1} - 1}{M - i +1} - \binom{K - c_i}{M - i + 1},
\end{equation*}
where the last equality obtained by~\eqref{hs_identity}, which also completes the derivation.

\myparagraph{Measured comparison (with $G=1$).} At FLUX scale ($M=219$, one image) the table already costs $61.1$\,s to build at $K=2^{14}$, but $1386.8$\,s ($\approx 23$\,min) at $K=2^{18}=d$, occupying $0.32$ and $8.10$\,GiB on disk respectively -- and, since each entry's bit-length grows with both $K$ and $M$, the table size grows worse than linearly in $M$: extrapolating to the larger $M$ used elsewhere in our FLUX grid projects tables of hundreds of gigabytes to several terabytes, one of which would have to be rebuilt from scratch for every distinct $(K,M)$ pair. Our recurrence pays none of this, at a per-call cost that is comparable to or faster than a table lookup once the table already exists (Table~\ref{tab:bitstream_head_to_head}).

\begin{table}[ht]
\centering
\small
\begin{tabular}{llrrrr}
\toprule
$K$ & Mechanism & One-time setup (s) & Compress (s) & Decompress (s) & Total (s) \\
\midrule
$2^{14}$ & Table (\citealp{vaisman2026turbo}) & 61.1 & 19.05 & 29.75 & 48.80 \\
$2^{14}$ & Recurrence (ours) & 0 & 18.06 & 27.68 & 45.73 \\
$2^{18}=d$ & Table (\citealp{vaisman2026turbo}) & 1386.8 & 37.09 & 38.14 & 75.23 \\
$2^{18}=d$ & Recurrence (ours) & 0 & 33.81 & 34.69 & 68.50 \\
\bottomrule
\end{tabular}
\caption{\textbf{Turbo-DDCM round-trip bitstream cost at $M=219$ (FLUX, with $d=2^{18}$).} The per-call cost is close between the two mechanisms; the one-time table build is what the original protocol pays fresh for every $(K,M)$ it is run at, and which the recurrence never pays.}
\label{tab:bitstream_head_to_head}
\end{table}

Standalone rank/unrank timings make the crossover between the two table-free mechanisms explicit (Table~\ref{tab:hockey_stick_crossover}). Hockey-stick evaluation is substantially cheaper when $M$ is small because it touches only the selected indices, whereas the recurrence performs an $O(K)$ sweep. As $M$ grows, repeated arbitrary-precision binomial evaluations make hockey-stick decoding increasingly expensive, while the recurrence grows much more slowly. We therefore use hockey-stick coding for the low-$M$ video regime and the recurrence for the broader image sweeps.

\begin{table}[ht]
\centering
\small
\resizebox{\linewidth}{!}{%
\begin{tabular}{lrrrrrr}
\toprule
$M$ & Recurrence rank (s) & Recurrence unrank (s) & Hockey-Stick rank (s) & Hockey-Stick unrank (s) & Rank ratio & Unrank ratio \\
\midrule
43   & 0.034 & 0.049 & 0.0001 & 0.0008  & 0.004$\times$ & 0.02$\times$ \\
67   & 0.033 & 0.050 & 0.0003 & 0.0022  & 0.008$\times$ & 0.04$\times$ \\
84   & 0.039 & 0.055 & 0.0004 & 0.0035  & 0.01$\times$  & 0.06$\times$ \\
128  & 0.041 & 0.060 & 0.0010 & 0.0091  & 0.02$\times$  & 0.15$\times$ \\
219  & 0.051 & 0.071 & 0.0034 & 0.0320  & 0.07$\times$  & 0.45$\times$ \\
373  & 0.063 & 0.083 & 0.0117 & 0.1153  & 0.19$\times$  & 1.4$\times$  \\
638  & 0.088 & 0.107 & 0.0457 & 0.4502  & 0.52$\times$  & 4.2$\times$  \\
1063 & 0.120 & 0.152 & 0.1709 & 1.6447  & 1.4$\times$   & 10.8$\times$ \\
1559 & 0.161 & 0.190 & 0.4594 & 4.4245  & 2.9$\times$   & 23.3$\times$ \\
2697 & 0.247 & 0.285 & 1.9003 & 18.291  & 7.7$\times$   & 64.1$\times$ \\
4641 & 0.352 & 0.399 & 7.8415 & 75.405  & 22.3$\times$  & 189.2$\times$ \\
\bottomrule
\end{tabular}%
}
\caption{\textbf{Recurrence vs.\ hockey-stick, at $K=d=2^{18}$ standalone rank/unrank cost, without model calls.} Ratios are hockey-stick over recurrence.}
\label{tab:hockey_stick_crossover}
\end{table}

\clearpage

%% file: supplementary/codebook_cost.tex
\section{Cost of an explicit codebook}
\label{app:codebook_cost}

This appendix quantifies the dense-codebook bottleneck discussed in Sec.~\ref{sec:bg:turbo}. Computing the correlations in~\eqref{eq:thresholding} accesses all $Kd$ codebook entries, so at every bit-carrying denoising step $\boldsymbol{C}_t$ must either be stored in memory or regenerated in tiles. Reconstruction of $\mathbf{z}_t^\star$ accesses only the $M$ selected columns, but those columns originate from the same dense codebook.

In \texttt{fp16}, forming $\mathbf{u}_t=\boldsymbol{C}_t^\top\mathbf{r}_t$ requires approximately $2Kd$ FLOPs while reading $2Kd$ bytes. The operation is therefore strongly bandwidth-sensitive. More importantly, the codebook itself occupies $2Kd$ bytes: approximately $0.5$\,GiB for SD~2.1 with $K=d=2^{14}$, $128$\,GiB for FLUX with $K=d=2^{18}$, and $167$\,GiB for a CF video chunk with $K=d=299{,}520$. At the larger scales, the complete codebook cannot reside alongside the backbone and must instead be regenerated and processed in blocks, increasing runtime.

Capping $K$ reduces this cost but restricts the span of the available atoms. For $K<d$, components of the residual outside the codebook span cannot be represented, regardless of the number of selected atoms. Consistent with this limitation, our FLUX experiments show lower quality for the capped $K=2^{14}$ Turbo-DDCM codebook than for the complete $K=d$ codebook at matched bitrate (Sec.~\ref{sec:experiments}). Thus, explicit Gaussian codebooks impose a memory-fidelity tradeoff that becomes increasingly restrictive as latent dimension grows.

%% file: supplementary/overcomplete_hadamard.tex
\section{Overcomplete Hadamard dictionary}
\label{app:overcomplete_hadamard}

This appendix provides the full construction and rate-matching details for the overcomplete extension of Sec.~\ref{sec:method:overcomplete}, evaluated with $B=8$ bases in the streaming-video experiments of Sec.~\ref{sec:vid_streaming_app}. In this section only, for convenience, we denote the padded Hadamard support to the next power-of-two dimension as $\tilde{d}$.

\myparagraph{Padded and cropped geometry.}
A CF latent chunk has true flattened dimension $d=299{,}520$, whereas the FWHT operates at the next supported power-of-two dimension $\tilde{d}=2^{19}$. Let $\mathbf P\in\{0,1\}^{d\times \tilde{d}}$ denote cropping to the true latent coordinates, so $\mathbf P^\top$ performs zero-padding and $\mathbf P\mathbf P^\top=\mathbf I_d$. In the padded domain, each $\mathbf D_{t,b}\mathbf H\in\mathbb R^{\tilde{d}\times \tilde{d}}$ is exactly orthogonal. The corresponding effective dictionary after cropping is
\begin{equation}
\label{eq:oc-dictionary}
\mathbf{A}_{t,b} = \mathbf{P}\mathbf{D}_{t,b}\mathbf{H} \in \mathbb{R}^{d\times \tilde{d}}, \qquad \mathbf{u}_{t,b} = \mathbf{A}_{t,b}^\top\mathbf{r}_t = \mathbf{H}\mathbf{D}_{t,b}\mathbf{P}^\top\mathbf{r}_t,
\end{equation}
for $b=1,\dots,B$, with $\mathbf{D}_{t,b}=\diag(\boldsymbol{\varepsilon}_{t,b})$ drawn from the shared seed exactly as in Sec.~\ref{sec:method:selection}. The cropped dictionary has orthonormal rows, $\mathbf A_{t,b}\mathbf A_{t,b}^\top=\mathbf I_d$, but not orthonormal columns when $\tilde{d}>d$. Thus, the exact single-basis thresholding statement in Sec.~\ref{sec:method:selection} applies to the padded-domain objective; after cropping, the effective dictionary is overcomplete even before multiple randomized bases are concatenated.

\myparagraph{Selection and reconstruction.} For $b=1,\dots,B$ let $\boldsymbol{\alpha}^\star_{t,b}\in\{-1,0,+1\}^{\tilde{d}}$ collect basis $b$'s coefficients. The encoder keeps the $M$ largest-magnitude entries across the concatenated scores $\{\mathbf{u}_{t,1},\dots,\mathbf{u}_{t,B}\}$, recording for each the basis it came from:
\begin{equation}
\label{eq:oc-support}
\mathcal{I}_t^{\mathrm{oc}} := \TopM\big(\{|(\mathbf{u}_{t,b})_j| : b\in[B],\, j\in[\tilde{d}]\}\big), \qquad (\boldsymbol{\alpha}^\star_{t,b})_j = \begin{cases}\sign\big((\mathbf{u}_{t,b})_j\big), & (b,j)\in\mathcal{I}_t^{\mathrm{oc}}, \\ 0, & \text{otherwise,}\end{cases}
\end{equation}
where $[B]:=\{1,\dots,B\}$ and $\mathcal{I}_t^{\mathrm{oc}}$ is now a set of (basis, coordinate) pairs, unlike the single-basis support $\mathcal{I}_t$ of Sec.~\ref{sec:method:selection}. The decoder regenerates every $\mathbf{D}_{t,b}$ from the shared seed and reconstructs
\begin{equation}
\label{eq:oc-noise}
\mathbf{y}_t = \sum_{b=1}^{B}\mathbf{P}\mathbf{D}_{t,b}\mathbf{H}\boldsymbol{\alpha}^\star_{t,b}, \qquad \mathbf{z}_t^\star = \mathbf{y}_t/\std(\mathbf{y}_t).
\end{equation}

\myparagraph{Hierarchical top-$M$.}
A full sort over all $B\tilde{d}$ scores is unnecessary when only the global top $M$ are required. It is sufficient to retain each basis's local top $M$ scores and then select the global top $M$ from the resulting $BM$ candidates.
\begin{proposition}
\label{prop:hierarchical-topm}
Let $\mathcal{I}_t^{\mathrm{local}}$ be the union of each basis's own top-$M$ indices, $|\mathcal{I}_t^{\mathrm{local}}| = BM$. The global top-$M$ of $\{|(\mathbf{u}_{t,b})_j|\}_{b,j}$ is exactly the top-$M$ of $\{|(\mathbf{u}_{t,b})_j| : (b,j) \in \mathcal{I}_t^{\mathrm{local}}\}$.
\end{proposition}
\begin{proof}
No basis contributes more than $M$ entries to a global set of size $M$, so discarding a basis's entries below its own top-$M$ cannot discard a global top-$M$ member.
\end{proof}
This turns a sort over $B\tilde{d}$ scores into $B$ local top-$M$ searches plus one search over $BM \ll B\tilde{d}$ candidates, and batches reconstruction across the active bases rather than looping over them.

\myparagraph{Bitstream and rate matching.}
The overcomplete dictionary contains $B\tilde{d}$ candidate atoms, so comparing it with the single-basis dictionary at the same $M$ would allocate a larger bitrate budget to the overcomplete variant. We therefore compare the two at matched slot length. Using the grouped protocol of App.~\ref{app:bitstream_protocol}, for the overcomplete variant, each randomized basis forms one group ($G=B$, $K_G=\tilde{d}$). For each single-basis configuration $M_r$, we choose the overcomplete budget $M_{\mathrm{oc}}$ whose grouped slot length is closest to the single-basis cost $\lceil\log_2\binom{\tilde{d}}{M_r}\rceil+M_rC$. We note that since the realized per-basis counts $(m_1,\ldots,m_B)$ are unknown before the actual support is selected, we estimate the grouped slot length during rate matching using the balanced allocation $m_b\in\{\lfloor M_{\mathrm{oc}}/B\rfloor,\lceil M_{\mathrm{oc}}/B\rceil\}$. Table~\ref{tab:oc-rate-matching} lists the matched configuration pairs used in Sec.~\ref{sec:vid_streaming_app}.

\begin{table}[h]
\centering
\small
\begin{tabular}{rrrr}
\toprule
Reference \oursshort{} $M_r$ & Overcomplete $M_{\mathrm{oc}}$ & \oursshort{} bits/slot & Overcomplete bits/slot \\
\midrule
16    & 13   & 276    & 272    \\
32    & 26   & 523    & 523    \\
64    & 52   & 984    & 985    \\
128   & 103  & 1{,}844 & 1{,}843 \\
256   & 204  & 3{,}436 & 3{,}438 \\
512   & 402  & 6{,}365 & 6{,}367 \\
1{,}024 & 789  & 11{,}710 & 11{,}713 \\
1{,}280 & 979     & 14{,}226 & 14{,}225 \\
1{,}536 & 1{,}168 & 16{,}668 & 16{,}668 \\
2{,}048 & 1{,}541 & 21{,}375 & 21{,}370 \\
\bottomrule
\end{tabular}
\caption{\textbf{Overcomplete dictionary rate-matched configurations at $B=8$, $\tilde{d}=2^{19}$.} The overcomplete dictionary always needs strictly fewer atoms than the reference it matches, since each of its indices costs more to name.}
\label{tab:oc-rate-matching}
\end{table}

\clearpage

%% file: supplementary/algorithms.tex
\section{\oursshort{} pseudo-code}
\label{app:algorithms}
Algorithms~\ref{alg:compression} and~\ref{alg:decompression} summarize the single-trajectory compression and decompression shared by the image and video formulations, over one target $\mathbf{x}_{\mathrm{src}}$. They omit the outer streaming-video segmentation and KV-cache refresh logic of Sec.~\ref{sec:vid_streaming_app}, which is specified separately in App.~\ref{app:video_experiment_configurations}. The operator $\mathrm{opt}(\mathbf{A}, \mathbf{x}, \hat{\mathbf{x}}, M)$ denotes the closed-form top-$M$ threshold and coefficient quantization using the implicit dictionary $\mathbf{A}$ of Sec.~\ref{sec:method:selection}; $\mathbf{A}$ is applied through the FWHT and is never materialized explicitly.

\begin{algorithm}[h]
\caption{\oursshort{} -- Compression}
\label{alg:compression}
\begin{algorithmic}[1]
\Require Target $\mathbf{x}_{\mathrm{src}}$, trajectory length $T$, tail length $N$ ($N=0$ for the consistency formulation), atoms per step $M$, formulation $\in\{\mathrm{DDPM}, \mathrm{consistency}\}$, shared seed $s$.
\Ensure Compressed representation $E(\mathbf{x}_{\mathrm{src}})$.
\State $E \gets \emptyset$ \Comment{Initialize the compressed representation}
\If{formulation $=$ DDPM}
    \State $\mathbf{x}_T \gets \mathrm{rand\_normal}(s)$ \Comment{$\mathbf{x}_T$ generation; seed $s$ shared with decoder}
\Else \Comment{Bootstrap selection counts as the first of the $T$ bit-carrying consistency steps}
    \State $\boldsymbol{\alpha} \gets \mathrm{opt}(\mathbf{A}_{T+1}, \mathbf{x}_{\mathrm{src}}, \mathrm{mean}(\mathbf{x}_{\mathrm{src}}), M)$ \Comment{Closed-form threshold, Eq.~\eqref{eq:thresholding}; proxy bootstrap against the source mean, App.~\ref{app:sampling_formulations}}
    \State $\mathbf{x}_T \gets \mathbf{A}_{T+1}\boldsymbol{\alpha} / \std(\mathbf{A}_{T+1}\boldsymbol{\alpha})$ \Comment{Eq.~\eqref{eq:std-norm}}
    \State $E \gets E \,\|\, \mathrm{pack}(\boldsymbol{\alpha})$ \Comment{Lexicographic rank \& sign bits, App.~\ref{app:bitstream_protocol}; concatenated to the compressed representation}
\EndIf
\For{$t = T, \ldots, N+2$}
    \State $\hat{\mathbf{x}}_{0|t} \gets$ $D(\mathbf{x}_t, t)$ \Comment{Denoiser activation}
    \State $\boldsymbol{\alpha} \gets \mathrm{opt}(\mathbf{A}_t, \mathbf{x}_{\mathrm{src}}, \hat{\mathbf{x}}_{0|t}, M)$ \Comment{Closed-form threshold against the residual, Eq.~\eqref{eq:thresholding}}
    \State $\mathbf{z}^\star \gets \mathbf{A}_t\boldsymbol{\alpha} / \std(\mathbf{A}_t\boldsymbol{\alpha})$ \Comment{Rescaled to unit variance, Eq.~\eqref{eq:std-norm}}
    \State $\mathbf{x}_{t-1} \gets \boldsymbol{\mu}_t(\mathbf{x}_t) + \sigma_t\mathbf{z}^\star$ \Comment{Eq.~\eqref{eq:ddpm}}
    \State $E \gets E \,\|\, \mathrm{pack}(\boldsymbol{\alpha})$ \Comment{Lexicographic rank \& sign bits, App.~\ref{app:bitstream_protocol}; concatenated to the compressed representation}
\EndFor
\State \Return $E(\mathbf{x}_{\mathrm{src}})$
\end{algorithmic}
\end{algorithm}

\begin{algorithm}[h]
\caption{\oursshort{} -- Decompression}
\label{alg:decompression}
\begin{algorithmic}[1]
\Require Compressed representation $E$, trajectory length $T$, tail length $N$, formulation $\in\{\mathrm{DDPM}, \mathrm{consistency}\}$, shared seed $s$.
\Ensure Decompressed target $\tilde{\mathbf{x}}_0$.
\If{formulation $=$ DDPM}
    \State $\tilde{\mathbf{x}}_T \gets \mathrm{rand\_normal}(s)$ \Comment{identical to the encoder's $\mathbf{x}_T$, same shared seed $s$}
\Else
    \State $\boldsymbol{\alpha} \gets \mathrm{unpack}(E)$ \Comment{Reads the bootstrap support \& signs the encoder packed}
    \State $\tilde{\mathbf{x}}_T \gets \mathbf{A}_{T+1}\boldsymbol{\alpha} / \std(\mathbf{A}_{T+1}\boldsymbol{\alpha})$ \Comment{Identical to the encoder}
\EndIf
\For{$t = T, \ldots, N+2$}
    \State $\hat{\mathbf{x}}_{0|t} \gets$ $D(\tilde{\mathbf{x}}_t, t)$ \Comment{Denoiser activation}
    \State $\boldsymbol{\alpha} \gets \mathrm{unpack}(E)$ \Comment{Reads this step's support \& signs from the bitstream}
    \State $\mathbf{z}^\star \gets \mathbf{A}_t\boldsymbol{\alpha} / \std(\mathbf{A}_t\boldsymbol{\alpha})$ \Comment{Rescaled to unit variance, Eq.~\eqref{eq:std-norm}}
    \State $\tilde{\mathbf{x}}_{t-1} \gets \boldsymbol{\mu}_t(\tilde{\mathbf{x}}_t) + \sigma_t\mathbf{z}^\star$ \Comment{Eq.~\eqref{eq:ddpm}}
\EndFor
\If{formulation $=$ DDPM}
    \State perform $N+1$ deterministic DDIM steps \Comment{Eq.~\eqref{eq:mu_ddim}}
\Else
    \State perform the final deterministic sampling step \Comment{Eq.~\eqref{eq:mu_consistency}}
\EndIf
\State \Return $\tilde{\mathbf{x}}_0$
\end{algorithmic}
\end{algorithm}

See Secs.~\ref{sec:experiments} and~\ref{sec:vid_streaming_app}, together with Apps.~\ref{app:image_experiment_configurations} and~\ref{app:video_experiment_configurations}, for the configurations used in practice.

\clearpage